\documentclass{article} % For LaTeX2e
\usepackage[final]{colm2026_conference}
\usepackage{amsmath}
\usepackage[inkscapelatex=false]{svg}
\usepackage{algorithm}
\usepackage{microtype}
\usepackage{hyperref}
\usepackage{url}
\usepackage{booktabs}
\usepackage{xspace}
\usepackage{multirow}
\usepackage{algpseudocode}
\usepackage{subcaption}
\usepackage{wrapfig}
\usepackage[table]{xcolor}

\usepackage{amsthm}
\newtheorem{lemma}{Lemma}
\newtheorem{theorem}{Theorem}
\newtheorem{proposition}{Proposition}

\theoremstyle{definition}

\theoremstyle{remark}
\newtheorem{remark}{Remark}
\usepackage{amsmath,amsfonts,bm}

\def\eqref#1{equation~\ref{#1}}
\def\1{\bm{1}}

\DeclareMathAlphabet{\mathsfit}{\encodingdefault}{\sfdefault}{m}{sl}
\SetMathAlphabet{\mathsfit}{bold}{\encodingdefault}{\sfdefault}{bx}{n}

\newcommand{\E}{\mathbb{E}}

\newcommand{\R}{\mathbb{R}}

\usepackage{xspace}
\DeclareRobustCommand{\ours}{\texttt{DARE}\xspace}
\DeclareRobustCommand{\kours}{\texttt{DARE-KV}\xspace}
\DeclareRobustCommand{\oours}{\texttt{DARE-O}\xspace}

\usepackage{lineno}

\definecolor{darkblue}{rgb}{0, 0, 0.5}
\hypersetup{colorlinks=true, citecolor=darkblue, linkcolor=darkblue, urlcolor=darkblue}

\title{\ours: \underline{D}iffusion Language Model \underline{A}ctivation \underline{R}euse for \underline{E}fficient Inference}

\author{Natalia Frumkin, Bokun Wang, Hung-Yueh Chiang, \& Diana Marculescu\\
Chandra Family Department of Electrical and Computer Engineering\\
The University of Texas at Austin\\
\texttt{\{nfrumkin,dianam\}@utexas.edu} \\
\And
Chi-Chih Chang, Mohamed S. Abdelfattah \\
Cornell University \\
}

\begin{document}

\ifcolmsubmission
\linenumbers
\fi

\maketitle

\begin{abstract}
Diffusion Large Language Models (dLLMs) have recently emerged as a compelling alternative to auto-regressive (AR) language models, offering a larger expressive field and significant potential for parallel generation and inference speedup. 
However, existing open-source dLLMs remain in an early stage of development, lagging behind their AR counterparts in both efficiency and task quality. 
In this work, we identify and exploit an underexplored property of dLLMs: the high degree of \textit{token-wise redundancy} in their bi-directional self-attention layers. 
We observe that self-attention activations are strongly correlated across tokens, and that the temporal change in query representations can be used to predict redundancy in the corresponding key, value, and output activations. 
Building on this insight, we introduce DARE (dLLM Activation Reuse for Efficient inference) and its two complementary reuse mechanisms: \kours, which selectively reuses cached key–value (KV) activations, and \oours, which reuses output activations to reduce redundant computation while preserving model quality.

Empirically, \ours achieves up to a $1.20\times$ reduction in per-layer latency and reuses as much as 87\% of blockwise attention activations, all with negligible quality degradation across reasoning and code-generation benchmarks. Both reuse mechanisms, \kours and \oours , maintain highly competitive generation quality with average performance drops of only 2.0\% and 1.2\%, respectively. When combined with existing acceleration techniques such as prefix caching and Fast-dLLM, our method delivers additive performance gains without requiring any model re-training. These results highlight token-wise reuse as a principled and effective strategy for improving the computational efficiency of diffusion-based LLMs, enabling scalable inference while preserving generation fidelity. Code is available at \url{https://github.com/enyac-group/DARE}.

\end{abstract}

\section{Introduction}

Diffusion Large Language Models (dLLMs) have recently emerged as an exciting alternative to conventional auto-regressive (AR) transformers for text generation, offering bidirectional context modeling, flexible decoding, and significant potential for parallel inference~\citep{li2022diffusionLM, nie2025large}. 
By predicting complete sequences through iterative denoising steps rather than left-to-right token emission, dLLMs relax the causal constraint inherent to AR models and open a promising path toward high-throughput generation. 
However, despite their theoretical advantages, current open-source dLLMs remain computationally expensive and memory-intensive, with inference costs dominated by repeated evaluation of large bi-directional self-attention layers across diffusion timesteps. 
These inefficiencies limit their scalability to long contexts and hinder their practical deployment compared to optimized AR counterparts such as Llama~\citep{touvron2023llama} or Mixtral~\citep{jiang2024mixtral}. In Figure~\ref{fig:cross-layer-cache}, we highlight a key distinction between AR-LLMs and dLLMs. Unlike in LLaMA, the value cache in diffusion models exhibits markedly different similarity patterns across layers. For LLaDA-8B at timesteps 0, 127, and 256, we observe consistently low intra-layer similarity, which hinders the direct application of techniques that rely on sharing key–value states across layers. This behavior motivates a deeper examination of how key–value caches can be more effectively characterized and exploited in dLLMs.

\begin{figure}[t]
    \centering
    \includegraphics[width=0.9\linewidth]{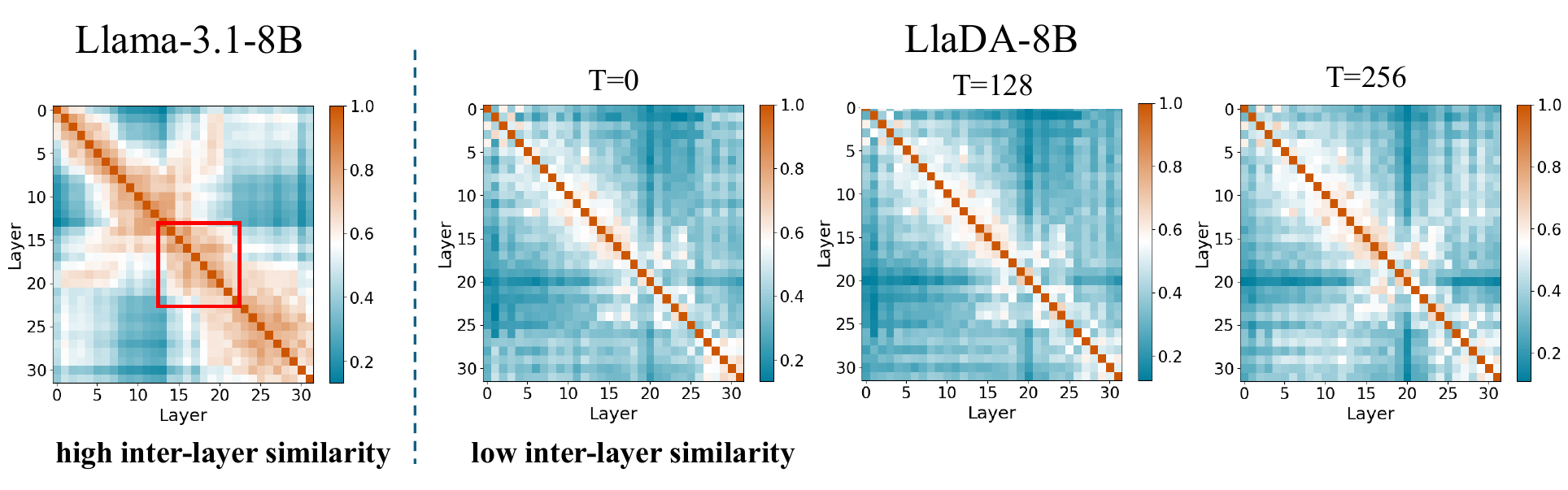}
    \caption{\textbf{Cross-layer value-cache similarity.} LLaDA-8B shows low redundancy across layers (vs.\ Llama-3.1-8B) at diffusion timesteps 0/127/255. Entry $(i,j)$ denotes cosine similarity between value caches at layers $i$ and $j$.}
    \label{fig:cross-layer-cache}
\end{figure}

A key insight of this work is that dLLMs exhibit a high degree of \textit{temporal token-wise redundancy} in their intermediate representations as shown on the left of Figure \ref{fig:method}. 
Across successive diffusion steps, many tokens undergo only minor changes in their query, key, or value embeddings, suggesting that large portions of self-attention computation could be reused instead of recomputed. 
We empirically observe that self-attention activations are highly correlated across timesteps and that query drift, \textit{i.e.}, the change in a token’s query representation between steps, can be used to reliably predict which tokens can safely reuse their cached activations.
This observation motivates the central question of our work: \textit{Can we exploit representational redundancy in diffusion LLMs to reduce attention compute without sacrificing generation quality?}

To answer this, we introduce \kours\ and \oours, two complementary methods for token-wise activation reuse in dLLMs. 
Our framework extends the standard diffusion inference pipeline with (1) \kours, which selectively reuses cached key–value (KV) activations, and (2) \oours, which reuses output activations when representational changes fall below a learned threshold. 
Together, these mechanisms enable efficient, training-free attention reuse that reduces redundant computation while preserving model fidelity.

These \textit{plug-and-play} methods require no retraining, gradient updates, or model reparameterization, and can be composed seamlessly with existing acceleration techniques such as prefix caching~\citep{ma2025dkv}, parallel decoding~\citep{chen2025dparallel}, and Fast-dLLM~\citep{wu2025fast}. 
Together, they provide a flexible means of reducing redundant attention computation while preserving output generation quality.

Our experiments show that selective reuse achieves up to a $1.20\times$ reduction in layer-level latency and reuses, on average, 60–75\% of attention activations across models and tasks, with negligible accuracy degradation. On average, \kours and \oours incur only 2.0\% and 1.2\% performance drops, respectively, while maintaining competitive generation quality. Analyses further reveal that reuse effectiveness scales with sequence length, correlates strongly with per-layer redundancy, and complements existing inference optimizations.

\begin{figure*}
    \centering
    \includegraphics[width=\linewidth]{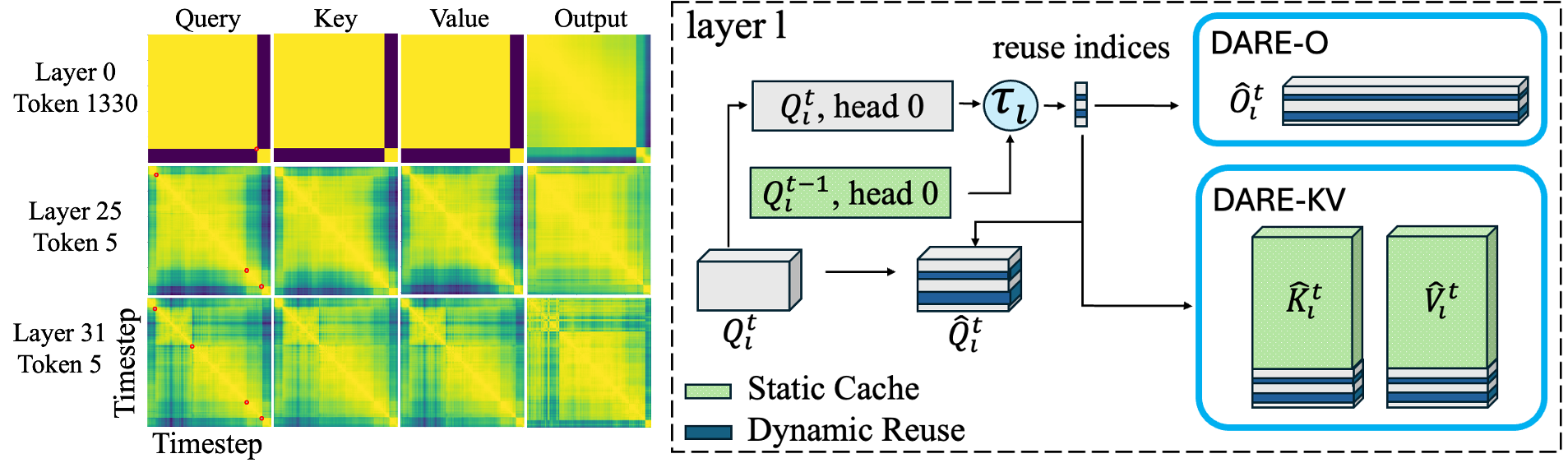}
    \caption{
\textbf{Left:} Along the token dimension, query, key, value, and output activations exhibit highly similar temporal similarity patterns. Each heatmap entry $(i,j)$ denotes the similarity between activations at timesteps $i$ and $j$. Red markers highlight sharp drops in similarity, indicating timesteps at which cached activations should be refreshed. Notably, the overall similarity structure is nearly identical across all activation types within a given self-attention layer. 
\textbf{Right:} Reuse indices computed from a single query head are sufficient to guide the dynamic reuse of the corresponding key, value, and output activations within the same layer. Specifically, for layer $\ell$, we compute similarity scores between head 0’s query tensor at timestep $t$ and that at timestep $t\!-\!1$, and apply a threshold to determine the token indices reused in downstream key/value (KV) or output (O) activations.
}
    \label{fig:method}
\end{figure*}

In summary, our work makes the following contributions:
\begin{itemize}
\item We identify and characterize \emph{token-wise temporal stability} in diffusion language models, a structural property arising from discrete token representations under bi-directional attention, and show that query drift serves as an effective proxy for activation stability.
\item We propose \kours and \oours, two efficient, training-free mechanisms that leverage this stability to selectively \emph{bypass redundant attention computation} via key–value and output reuse. Unlike caching-based approaches that store and reuse previously computed activations, our method determines when activations need not be recomputed at all, recomputing only when token representations change significantly.
\item We provide a theoretical analysis of \ours, showing that over $T$ diffusion steps, the cumulative error introduced by selective recomputation can be bounded as a function of the reuse threshold.
\item We demonstrate consistent accuracy preservation across multiple dLLMs and benchmarks, and analyze how computation reuse interacts with block size, reuse threshold, and layer depth.
\end{itemize}

By identifying a source of redundancy specific to diffusion over discrete token spaces and exploiting it to avoid unnecessary recomputation, this work provides a practical and complementary approach to improving the efficiency of dLLM inference.

% Is this the related work section?
\section{Related Work}
\paragraph{Diffusion language models and block-wise generation}
Diffusion language models (dLLMs) extend diffusion to discrete token sequences, generating text via iterative denoising rather than left-to-right decoding.
Prior work~\citep{li2022diffusionLM, ye2025dream, xie2025dream, nie2025large} shows strong performance with fully bidirectional context, but at high computational cost: each step recomputes all tokens, creating substantial redundancy compared to autoregressive models with caching.

Recent efforts improve efficiency and scalability.
Seed Diffusion~\citep{song2025seed} amortizes computation via latent seeds, while Gemini Diffusion~\citep{deepmind2025geminidiffusion} combines diffusion refinement with parallel decoding and prefix caching.
Block-wise formulations further bridge diffusion and autoregressive generation:
LLaDA~\citep{nie2025large} and Block Diffusion~\citep{arriola2025block} perform autoregressive block generation with parallel denoising within blocks, enabling partial KV caching and improved throughput.

Despite these advances, dLLMs still recompute many token representations that change little across steps.
Because tokens persist across timesteps, dLLMs exhibit token-wise temporal redundancy, motivating our approach to selectively avoid redundant computation.

\paragraph{Accelerating inference for dLLMs}
Inference in transformer-based large language models (LLMs) is computationally intensive due to repeated recomputation of prior Keys and Values at each timestep. A key technique for mitigating this cost is Key-Value (KV) caching, where the model stores the key and value tensors of each transformer layer after each decoding step, enabling reuse of previously computed activations~\citep{pope2023efficiently}.

In diffusion LLMs (dLLMs), however, the effectiveness of caching is fundamentally limited. Because dLLMs employ bi-directional attention and update all tokens at every timestep, cached activations cannot eliminate the need to recompute attention over the entire sequence. As a result, existing approaches primarily focus on improving efficiency through better caching policies, scheduling strategies, or token selection.

\textbf{Fast-dLLM}~\citep{wu2025fast} improves throughput by overlapping diffusion steps via parallel decoding and prefix caching, while \textbf{Fast-dLLM 2}~\citep{wu2025fast2} extends these ideas to training.
\emph{Our method is orthogonal}: these approaches change \emph{when} computation is executed and how it is scheduled across timesteps, whereas we reduce \emph{how much} computation is performed within each step by selectively skipping attention for stable tokens.

\textbf{DPad}~\citep{chen2025dpad} reduces computation by trimming the input sequence using suffix-windowing and distance-decay dropout.
\emph{Our method is orthogonal}: DPad modifies the \emph{input token set} and effectively prunes context, while our approach preserves the full sequence and instead avoids redundant recomputation based on intermediate activation stability.

\textbf{dKV-Cache}~\citep{ma2025dkv} dynamically retains or discards KV states based on their estimated contribution to future decoding steps, reducing memory and compute overhead.
\emph{Our method is orthogonal}: dKV-Cache operates at the level of \emph{what to store and reuse} (a caching policy), whereas we operate at the level of \emph{whether to recompute} specific token activations, without altering the cache itself.

\textbf{DParallel}~\citep{chen2025dparallel} improves efficiency through parallel decoding strategies that increase hardware utilization.
\emph{Our method is orthogonal}: DParallel improves \emph{throughput via parallelism}, while our approach reduces the \emph{intrinsic computational workload} by skipping redundant attention operations.

\paragraph{Our approach.}
Our method is fundamentally distinct from caching-based or scheduling-based approaches. Rather than deciding what to cache or when to execute computation, we perform \emph{fine-grained, token-level recomputation control} based on intermediate activations. Specifically, we observe that many token representations remain stable across diffusion steps, and leverage this to selectively skip attention computation for those tokens. This results in a \emph{surgical reduction of redundant computation within each step}, making our method complementary to existing techniques and readily composable with them. In Table \ref{tab:kvreuse_all}, we show that \ours integrates seamlessly with prefix caching and \cite{wu2025fast} in .

\section{\ours: Our Method for dLLM Attention Reuse}

\subsection{Layerwise dynamic thresholding}\label{sec:threshold}

\paragraph{Token drift scores.}
To enable fine-grain activation reuse throughout the diffusion process, we develop a dynamic thresholding mechanism that adapts to local representational drift, along with a layerwise reuse budget allocation strategy that prioritizes layers with higher reuse potential. We first define a drift score, $s(\mathbf{x}^t,\mathbf{x}^{t-1})$, based on the cosine similarity for a given intermediary activation $\mathbf{x}$ (such as query, key, value, or output) that quantifies its representational change across between two adjacent timesteps $(t-1,t)$:
\begin{align*}
    s(\mathbf{x}^{t}, \mathbf{x}^{t-1}) := 1 -  \frac{\langle \mathbf{x}^{t}, \mathbf{x}^{t-1} \rangle}{||\mathbf{x}^{t}||_2 ||\mathbf{x}^{t-1}||_2}. 
\end{align*}
Since our experimental observations (see Figure~\ref{fig:method}, left) indicate that the query, key, value, and output activations within each layer exhibit similar temporal coherence, we compute the drift score $s(\mathbf{q}_{i,\ell}^t, \mathbf{q}_{i,\ell}^{t-1})$ only for the query of each token $i$ in layer $\ell$ and use it to guide the thresholding of all other activations. 

\paragraph{Reuse rate allocation}
Given the per-token, per-layer drift scores, we can threshold activations with high temporal redundancy using a chosen quantile $\phi$. However, different layers may exhibit varying degrees of temporal redundancy and importance, making a single global reuse budget suboptimal. To address this, we introduce a layerwise reuse budget allocation strategy that distributes the total reuse budget across layers based on their representational drifts. We define the \emph{layerwise} drift scores\footnote{The layerwise drift score in (\ref{eq:layerwise_score}) is conceptually similar to the Block Influence (BI) score in \citet{men2024shortgpt}. However, BI measures intra-layer transformation, whereas our score measures the temporal coherence of query activations between two timesteps.} $\mathbf{s} = (s_1,\dotsc, s_L)^\top$ as the expected drift scores across timesteps and tokens:
\begin{equation}\label{eq:layerwise_score}
    s_{\ell}
    := \E_{t, i} s(\mathbf{q}_{i,\ell}^{t}, \mathbf{q}_{i,\ell}^{t-1}),\quad \ell\in\{1,\dotsc,L\},
\end{equation}
which could be estimated on-the-fly using a sampled calibration dataset. Intuitively, layers with more temporal redundancy should receive a larger reuse budget, while ensuring that no single layer consumes nearly the entire budget. As shown by \citet{lin2024modegpt}, the theoretically optimal layerwise quantiles $\boldsymbol{\phi}=(\phi_1,\dotsc,\phi_L)^\top$ that can achieve the goal above are given by:
\begin{align*}
\phi_\ell = (L\bar{\phi})\cdot \texttt{softmax}( - s_\ell / \epsilon ),\quad \ell\in\{1,\dotsc,L\},
\end{align*}
where $\bar{\phi}$ is is the average quantile and $\epsilon$ is the temperature parameter (both hyperparameters). Then, the threshold $\tau_\ell$ of layer $\ell$ is the drift score corresponding to the $\phi_\ell$-quantile.

\subsection{Dynamic KV reuse}\label{sec:kvreuse}
At timestep $t$, the model first retrieves the stored query activations 
from the previous diffusion step, 
$\{\Tilde{\mathbf{Q}}_{\ell}^{t-1}\}_{\ell=1}^L$, 
which capture per-token representations from the prior denoising iteration. 
For each self-attention layer $\ell$, we compute the per-token similarity 
score $s(\mathbf{q}_{i,\ell}^t, \mathbf{q}_{i,\ell}^{t-1})$ between the 
current and previous query activations of token $i$, as shown in Figure \ref{fig:method} (right). In practice, we only consider the score for head 0 as we observe that all heads share a similar reuse pattern.
This score measures the representational drift of each token where smaller values 
indicate that a token’s representation remains stable and can therefore 
be safely reused.

Based on the dynamic threshold $\tau_\ell$ introduced in 
Section~\ref{sec:threshold}, we form the reuse set
\begin{equation}
\mathcal{R}_\ell^t = 
\{\, i \mid s(\mathbf{q}_{i,\ell}^t, \mathbf{q}_{i,\ell}^{t-1}) 
   \le \tau_\ell \,\},
\end{equation}
which contains the indices of tokens whose activations exhibit minimal 
change between consecutive diffusion steps. 
Its complement $\bar{\mathcal{R}}_\ell^t = 
\{\, i \mid s(\mathbf{q}_{i,\ell}^t, \mathbf{q}_{i,\ell}^{t-1}) 
   > \tau_\ell \,\}$ 
corresponds to tokens that must be refreshed. 
We therefore reuse the cached key and value activations 
$(\mathbf{K}_{i,\ell}^{t-1}, \mathbf{V}_{i,\ell}^{t-1})$ 
for $i \in \mathcal{R}_\ell^t$, while recomputing new projections
\begin{equation}
\mathbf{k}_i^t = \mathbf{x}_i^t W_K, 
\qquad 
\mathbf{v}_i^t = \mathbf{x}_i^t W_V,
\end{equation}
only for $i \in \bar{\mathcal{R}}_\ell^t$. 
The resulting hybrid cache $(\mathbf{K}_\ell^t, \mathbf{V}_\ell^t)$ 
is then fed into the self-attention mechanism to produce the attention 
output $\mathbf{O}_{\ell}^t$ (\textit{i.e.}, the output of the multi-head attention 
block prior to the feed-forward sublayer). 
This selective recomputation allows \ours\ to exploit temporal redundancy 
in diffusion LLMs preserving accuracy for dynamic tokens while avoiding 
redundant key--value projections for stable ones.
Pseudo-code for this process is provided in 
Algorithm~\ref{alg:kvreuse}.

\subsection{Dynamic output reuse}\label{sec:oreuse}
Alternatively, we can apply reuse at the \emph{attention output} level 
rather than the key--value cache. 
In this variant, denoted \oours\ 
(Algorithm~\ref{alg:oreuse}), 
tokens in $\mathcal{R}_\ell^t$ directly reuse their previously computed 
attention outputs $\mathbf{O}_{i,\ell}^{t-1}$, 
while only those in $\bar{\mathcal{R}}_\ell^t$ are recomputed via the 
attention operation:
\begin{equation}
\mathbf{O}_{i,\ell}^{t} = 
\begin{cases}
\mathbf{O}_{i,\ell}^{t-1}, & \text{if } i \in \mathcal{R}_\ell^t, \\
\mathrm{Attention}(\mathbf{q}_{i,\ell}^t, 
                  \mathbf{K}_{\ell}^t, 
                  \mathbf{V}_{\ell}^t), 
& \text{if } i \in \bar{\mathcal{R}}_\ell^t.
\end{cases}
\end{equation}
This complementary mechanism targets redundancy in the final stage of 
each self-attention block, offering additional computational savings 
with minimal quality degradation. 
Together, \kours\ and \oours\ provide two complementary axes of activation 
reuse, within the attention cache and at the layer output enabling efficient 
diffusion inference without model retraining.

\begin{figure}[t]
\centering

\begin{minipage}[t]{0.48\linewidth}
\begin{algorithm}[H]
\caption{\kours for layer $\ell$ at step $t$}
\label{alg:kvreuse}
\begin{algorithmic}[1]
\Require Hidden states $\mathbf{x}_t$, previous query $\mathbf{Q}_{t-1}^0$, previous $\mathbf{K}_{t-1}, \mathbf{V}_{t-1}$, threshold $\tau_\ell$
\Ensure Updated outputs $\mathbf{O}_t$

\State $\mathbf{x}_t^{\text{norm}} \gets \mathrm{LayerNorm}(\mathbf{x_t})$
\State $\mathbf{Q}_t \gets W_Q \mathbf{x}_t^{\text{norm}}$

\State $\mathbf{s} \gets s(\mathbf{Q}_t^0, \mathbf{Q}_{t-1}^0)$
\State $\mathcal{R} \gets \{ i \mid s_i \le \tau_\ell \}$
\State $\bar{\mathcal{R}} \gets \{1,\dots,T\} \setminus \mathcal{R}$

\State $\mathbf{x}_{\bar{\mathcal{R}}} \gets \mathbf{x}_t^{\text{norm}}[\bar{\mathcal{R}}]$
\State $\mathbf{K}_{\bar{\mathcal{R}}} \gets W_K \mathbf{x}_{\bar{\mathcal{R}}}$, \quad $\mathbf{V}_{\bar{\mathcal{R}}} \gets W_V \mathbf{x}_{\bar{\mathcal{R}}}$
\State $\mathbf{K} \gets \mathbf{K}_{\bar{\mathcal{R}}} \cup \mathbf{K}_{t-1}[\mathcal{R}]$
\State $\mathbf{V} \gets \mathbf{V}_{\bar{\mathcal{R}}} \cup \mathbf{V}_{t-1}[\mathcal{R}]$
\State $\mathbf{O} \gets \mathrm{DotProductAttention}(\mathbf{Q}_t, \mathbf{K}, \mathbf{V})$

\State $\mathbf{O}_t \gets W_O \mathbf{O}$
\State \Return $\mathbf{O}_t$
\end{algorithmic}
\end{algorithm}
\end{minipage}
\hfill
\begin{minipage}[t]{0.48\linewidth}
\begin{algorithm}[H]
\caption{\oours for layer $\ell$ at step $t$}
\label{alg:oreuse}
\begin{algorithmic}[1]
\Require Hidden states $\mathbf{x}_t$, previous query $\mathbf{Q}_{t-1}^0$, previous outputs $\mathbf{O}_{t-1}$, threshold $\tau_\ell$
\Ensure Updated outputs $\mathbf{O}_t$

\State $\mathbf{x}_t^{\text{norm}} \gets \mathrm{LayerNorm}(\mathbf{x_t})$
\State $\mathbf{Q}_t \gets W_Q \mathbf{x}_t^{\text{norm}}$

\State $\mathbf{s} \gets s(\mathbf{Q}_t^0, \mathbf{Q}_{t-1}^0)$
\State $\mathcal{R} \gets \{ i \mid s_i \le \tau_\ell \}$
\State $\bar{\mathcal{R}} \gets \{1,\dots,T\} \setminus \mathcal{R}$

\State $\mathbf{Q}_{\bar{\mathcal{R}}} \gets \mathbf{Q}_t[\bar{\mathcal{R}}]$
\State $\mathbf{K} \gets W_K \mathbf{x}_t^{\text{norm}}$, \quad $\mathbf{V} \gets W_V \mathbf{x}_t^{\text{norm}}$
\State $\mathbf{O}_{\bar{\mathcal{R}}} \gets \mathrm{DotProductAttention}(\mathbf{Q}_{\bar{\mathcal{R}}}, \mathbf{K}, \mathbf{V})$

\State $\mathbf{O}_t \gets \mathbf{O}_{t-1}$
\State $\mathbf{O}_t[\bar{\mathcal{R}}] \gets \mathbf{O}_{\bar{\mathcal{R}}}$
\State $\mathbf{O}_t \gets W_O \mathbf{O}_t$
\State \Return $\mathbf{O}_t$
\end{algorithmic}
\end{algorithm}
\end{minipage}

\end{figure}

\section{Reuse error analysis of \ours}\label{sec:theory}

We analyze the impact of cumulative errors introduced by KV or O reuse at each timestep on the generated sequence. For clarity, we focus on \ours based on a single-layer transformer using single-head self-attention.
 For comparison, we construct a reference sequence $\mathbf{x}_{\text{full},i}^t$ (see Appendix~\ref{sec:full_seq} for details), whose marginal distribution is given by the softmax output of the \emph{full} dLLM (\textit{i.e.}, without KV or O reuse).

Note that the reused key/value/output activations for the $i$-th token ($i \in \mathcal{R}^t$) are actually computed at the most recent timestep before $t$ in which the token was not reused.  To quantify the ``staleness'' of the K-V or O vector, we define the consecutive reuse count $\delta_i^t(\tau)$ for token $i$ at step $t$:

$\bullet$ If $i\notin \mathcal{R}^t$ (refreshed), $\delta_i^t(\tau) = 0$.

$\bullet$ If $i\in \mathcal{R}^t$ (reused), $\delta_i^t(\tau) = \delta_i^{t-1}(\tau) + 1$.
 
Thus, for a reused token $i\in \mathcal{R}^t$, its key and value activations are in fact computed at step $t' = t - \delta_i^t$. We define the concatenated staleness vector $\Delta^t(\tau) := (\delta_1^t(\tau),\dotsc, \delta_B^t(\tau))^\top$. For a threshold of $\tau = 0$, it follows that $\delta_i^t(\tau) \equiv 0$.

We can establish the following theoretical result.
\begin{theorem}\label{thm:main}
After $T$ steps, the expected cumulative error of \kours or \oours can be upper bounded as follows:
\begin{align*}
\sum_{i=1}^B \E\|\mathbf{x}_{\text{full},i}^t - \mathbf{x}_i^t\|_2 \leq O\left(\sqrt{\tau}\sum_{t=1}^T G^{T-t}\E \|\Delta^{t-1}(\tau)\|_2\right),
\end{align*}
where $\tau$ is the reuse threshold and $G$ denotes the Lipschitz continuity constant (\textit{i.e.}, a measure of robustness) of the full dLLM with respect to the input sequence.
\end{theorem}
The proof of the theorem above is deferred to Appendix~\ref{sec:thm_proof}. Our result shows that the expected cumulative error of \kours or \oours is controlled by the reuse threshold $\tau$, the Lipschitz constant $G$ of the architecture, and the expected staleness of the reused activations (which depends on $\tau$ implicitly). 
An explicit upper bound for $G$ is provided in Appendix~\ref{sec:lip}.

\section{Results}
\begin{table}[h]
\small
\centering
\setlength{\tabcolsep}{4pt}
\caption{
\textbf{Latency–accuracy tradeoff on LLaDA (Fast-dLLM vs.\ +\oours).}
Latency is reported at sequence length 512, batch size 1.
Accuracy is Flexible Extract (\%).
}
\label{tab:latency_accuracy_llada}
\begin{tabular}{lcccc}
\toprule
\textbf{Method} & \textbf{Tok/s} $\uparrow$ & \textbf{TTLT} $\downarrow$ & \textbf{GSM8K} & \textbf{HumanEval} \\
\midrule
Fast-dLLM 
& 17.7 & 28975 
& \textbf{78.2} & 36.0 \\

+\oours (ours)
& \textbf{136.6} & \textbf{3747}
& 78.1 & \textbf{37.2} \\
\midrule
& \multicolumn{2}{c}{\textbf{7.7$\times$ speedup, 87\% latency reduction}} 
& \multicolumn{2}{c}{\textbf{+0.2 / -0.1 pts}} \\
\bottomrule
\end{tabular}
\end{table}

\paragraph{Setup}
We evaluate \kours and \oours on four reasoning and code benchmarks (GSM8K, HumanEval, Minerva Math, MBPP) across three dLLMs (Dream-7B, LLaDA-8B, LLaDA-1.5-8B).
We report reuse (\%) and Flexible Extract accuracy, a relaxed metric for semantic correctness.
Unless noted, we use default model settings with $\bar{\phi}=0.3$; latency is measured on an RTX A6000 (BF16).
Additional details are provided in Appendix~\ref{app:setup}.

\begin{wraptable}{r}{0.4\linewidth}
\vspace{-6pt}
\centering
\caption{\textbf{Accuracy–latency tradeoff.} DARE-O lowers TPOT with minor accuracy drop.}
\label{tab:tiny_tradeoff}
\begin{tabular}{lcc}
\toprule
\textbf{Method} & \textbf{GSM8K} & \textbf{TPOT} \\
\midrule
Fast-dLLM-v2 & \textbf{82.64} & 12.70 \\
+ DARE-O & 79.76 & \textbf{11.60} \\
\bottomrule
\end{tabular}
\vspace{-8pt}
\end{wraptable}

\paragraph{End-task accuracy and end-to-end performance}
We evaluate \ours under two realistic inference configurations:
(1) Prefix Caching with Parallel Decoding (P+C) and
(2) Fast-dLLM~\citep{wu2025fast}.

\begin{table*}[t]
\footnotesize
\centering
\setlength{\tabcolsep}{4pt}
\caption{
Evaluation of \ours\ on \textbf{Flexible Extract (Flex.)} performance across models and benchmarks.
Flex. values are percentages (rounded to three significant figures). 
Reuse \% denotes the fraction of key–values or output activations reused across diffusion steps.
\textbf{P+C} = Prefix Cache + Parallel Decoding. Bold indicates best in setting, underline is second best.
}
\label{tab:kvreuse_all}
\begin{tabular}{@{}llcccccccc@{}}
\toprule
\multirow{2}{*}{\textbf{Benchmark}} & \multirow{2}{*}{\textbf{Method}} 
& \multicolumn{2}{c}{\textbf{Dream}} 
& \multicolumn{2}{c}{\textbf{LLaDA}} 
& \multicolumn{2}{c}{\textbf{LLaDA-1.5}} \\
\cmidrule(lr){3-4} \cmidrule(lr){5-6} \cmidrule(lr){7-8}
 & & Reuse \% & Flex. (\%) & Reuse \% & Flex. (\%) & Reuse \% & Flex. (\%) \\
\midrule

% ================= GSM8K =================
\multirow{9}{*}{\begin{tabular}[c]{@{}l@{}}GSM8k\\(5-shot)\end{tabular}}
& Baseline & – & 75.1 & – & 77.6 & – & 77.5 \\
\cmidrule(lr){2-8}
& \textbf{P+C} & – & \textbf{75.1} & – & \underline{77.2} & – & \underline{80.3} \\
& \quad + \kours\ (ours) & 75.8 & \underline{74.5} & 38.7 & \textbf{78.2} & 15.0 & 79.3 \\
& \quad + \oours\ (ours) & 93.4 & 73.6 & 66.7 & 77.0 & 66.4 & \textbf{80.9} \\
\cmidrule(lr){2-8}
& \textbf{Fast-dLLM} & – & \textbf{73.5} & – & \underline{78.2} & – & \textbf{80.7} \\
& \quad + \kours\ (ours) & 58.2 & 72.9 & 29.2 & \textbf{78.5} & 26.7 & 80.4 \\
& \quad + \oours\ (ours) & 94.8 & \underline{73.8} & 84.8 & 78.1 & 84.9 & \textbf{80.7} \\
\midrule

% ================= HumanEval =================
\multirow{9}{*}{\begin{tabular}[c]{@{}l@{}}Human\\Eval\end{tabular}}
& Baseline & – & 53.1 & – & 42.1 & – & 43.3 \\
\cmidrule(lr){2-8}
& \textbf{P+C} & – & \textbf{54.9} & – & \underline{42.7} & – & \underline{39.0} \\
& \quad + \kours\ (ours) & 42.9 & 51.8 & 31.3 & \textbf{43.3} & 31.9 & \textbf{39.6} \\
& \quad + \oours\ (ours) & 92.9 & \underline{54.3} & 47.1 & \underline{42.7} & 46.2 & 38.4 \\
\cmidrule(lr){2-8}
& \textbf{Fast-dLLM} & – & 54.3 & – & \underline{36.0} & – & \textbf{36.0} \\
& \quad + \kours\ (ours) & 39.4 & \textbf{55.5} & 9.08 & 34.8 & 9.31 & \textbf{36.0} \\
& \quad + \oours\ (ours) & 94.4 & \underline{54.9} & 65.7 & \textbf{37.2} & 63.0 & 34.8 \\
\midrule

% ================= Minerva Math =================
\multirow{9}{*}{\begin{tabular}[c]{@{}l@{}}Minerva\\Math\\(4-shot)\end{tabular}}
& Baseline & – & 38.4 & – & 7.92 & – & 33.3 \\
\cmidrule(lr){2-8}
& \textbf{P+C} & – & \textbf{37.7} & – & 6.94 & – & \underline{32.0} \\
& \quad + \kours\ (ours) & 75.4 & \underline{37.2} & 37.9 & \underline{7.24} & 38.3 & \textbf{32.5} \\
& \quad + \oours\ (ours) & 91.5 & 36.9 & 61.3 & \textbf{7.42} & 60.8 & \underline{32.0} \\
\cmidrule(lr){2-8}
& \textbf{Fast-dLLM} & – & \textbf{37.1} & – & \underline{8.06} & – & \textbf{32.0} \\
& \quad + \kours\ (ours) & 56.0 & 35.9 & 28.4 & \textbf{8.60} & 28.2 & \underline{31.4} \\
& \quad + \oours\ (ours) & 93.2 & \underline{36.9} & 78.1 & 8.02 & 78.0 & 30.8 \\
\midrule

% ================= MBPP =================
\multirow{9}{*}{\begin{tabular}[c]{@{}l@{}}MBPP\\(3-shot)\end{tabular}}
& Baseline & – & 55.6 & – & 30.4 & – & 39.0 \\
\cmidrule(lr){2-8}
& \textbf{P+C} & – & \textbf{54.6} & – & \textbf{27.0} & – & \underline{38.0} \\
& \quad + \kours\ (ours) & 74.0 & \underline{54.4} & 34.5 & \underline{26.2} & 15.0 & \textbf{38.8} \\
& \quad + \oours\ (ours) & 90.8 & 51.2 & 63.1 & 25.6 & 65.9 & 37.8 \\
\cmidrule(lr){2-8}
& \textbf{Fast-dLLM} & – & \underline{53.0} & – & 25.8 & – & \textbf{35.8} \\
& \quad + \kours\ (ours) & 57.4 & \textbf{54.2} & 26.9 & \textbf{26.6} & 22.3 & 33.8 \\
& \quad + \oours\ (ours) & 92.9 & 51.8 & 83.7 & \underline{26.4} & 87.2 & \underline{35.0} \\
\bottomrule
\end{tabular}
\end{table*}

Table~\ref{tab:latency_accuracy_llada} summarizes the end-to-end latency--accuracy tradeoff on LLaDA.
At sequence length 512, \oours achieves a 7.7$\times$ increase in throughput and an 87\% reduction in latency compared to Fast-dLLM, while preserving task performance.
On GSM8K, accuracy remains effectively unchanged (78.2 $\rightarrow$ 78.1), and on HumanEval, performance slightly improves (36.0 $\rightarrow$ 37.2).
These results show that substantial system-level acceleration can be achieved with negligible impact on downstream quality.

In Table~\ref{tab:tiny_tradeoff}, we additionally evaluate \oours with Fast-dLLMv2~\citep{wu2025fast2}, an improved variant of Fast-dLLM that further optimizes decoding through more aggressive parallelism and scheduling.
In this setting, we apply \oours selectively to ensure stability: we skip the first $\ell{=}10$ layers, which are more sensitive to recomputation, and instead apply activation reuse to later layers where representations are more stable.
We also introduce a periodic \emph{refresh interval} of 2, recomputing activations every two diffusion steps to prevent error accumulation.

We report detailed accuracy in Table~\ref{tab:kvreuse_all} and system metrics in Table~\ref{tab:dare_latency_grouped}.
Across models and benchmarks, \ours achieves substantial reuse with near-baseline quality.
\kours reuses 50--75\% of KV states without degradation (e.g., 74--79\% GSM8K accuracy with up to 75.8\% reuse), and remains effective with Fast-dLLM (~58\% reuse, 72--80\% accuracy).

\oours, which reuses output activations, is more aggressive yet similarly robust, maintaining 77--81\% accuracy on GSM8K and MBPP with up to 85\% reuse.
Across benchmarks, both methods remain within 1--3\% of baseline ($\leq 2\%$ on Minerva), consistent across model scales.

These gains translate to end-to-end speedups (Table~\ref{tab:dare_latency_grouped}).
At batch size 16, \ours achieves up to 5.8$\times$ higher throughput (17.7 $\rightarrow$ 103.1 Tok/s at length 512) with corresponding reductions in latency, and remains effective at longer sequence lengths.

Full results across batch sizes are provided in Appendix~\ref{app:full_latency}.
Overall, \ours is robust and composable with techniques such as prefix caching and Fast-dLLM, delivering substantial efficiency gains with negligible quality loss.

\begin{table}[t]
\centering
\caption{
\textbf{End-to-end performance with \ours\ (batch size = 16).}
Tok/s = throughput (higher is better).
TTLT = total time-to-last-token (ms).
TPOT = time per output token (ms). "Base" is Fast-dLLM.
}
\label{tab:dare_latency_grouped}
\begin{tabular}{c|ccc|ccc|ccc}
\toprule
\textbf{Seq} &
\multicolumn{3}{c}{\textbf{Tok/s}} &
\multicolumn{3}{c}{\textbf{TTLT}} &
\multicolumn{3}{c}{\textbf{TPOT}} \\
\cmidrule(lr){2-4} \cmidrule(lr){5-7} \cmidrule(lr){8-10}
& Base & +\ours & ($\times$)
& Base & +\ours & (\%)
& Base & +\ours & (\%) \\
\midrule
256
& 17.7 & \textbf{27.0} & 1.5$\times$
& 14490 & \textbf{9497} & 34\%
& 56.6 & \textbf{37.1} & 34\% \\
512
& 17.7 & \textbf{103.1} & 5.8$\times$
& 28975 & \textbf{4964} & 83\%
& 56.6 & \textbf{9.7} & 83\% \\
1024
& 16.1 & \textbf{97.8} & 6.1$\times$
& 63764 & \textbf{10469} & 84\%
& 62.3 & \textbf{10.2} & 84\% \\
2048
& 58.5 & \textbf{59.4} & 1.0$\times$
& 35016 & \textbf{34460} & 2\%
& 17.1 & \textbf{16.8} & 2\% \\
\bottomrule
\end{tabular}
\end{table}

\begin{figure*}[h]
    \centering
    \includegraphics[width=\linewidth]{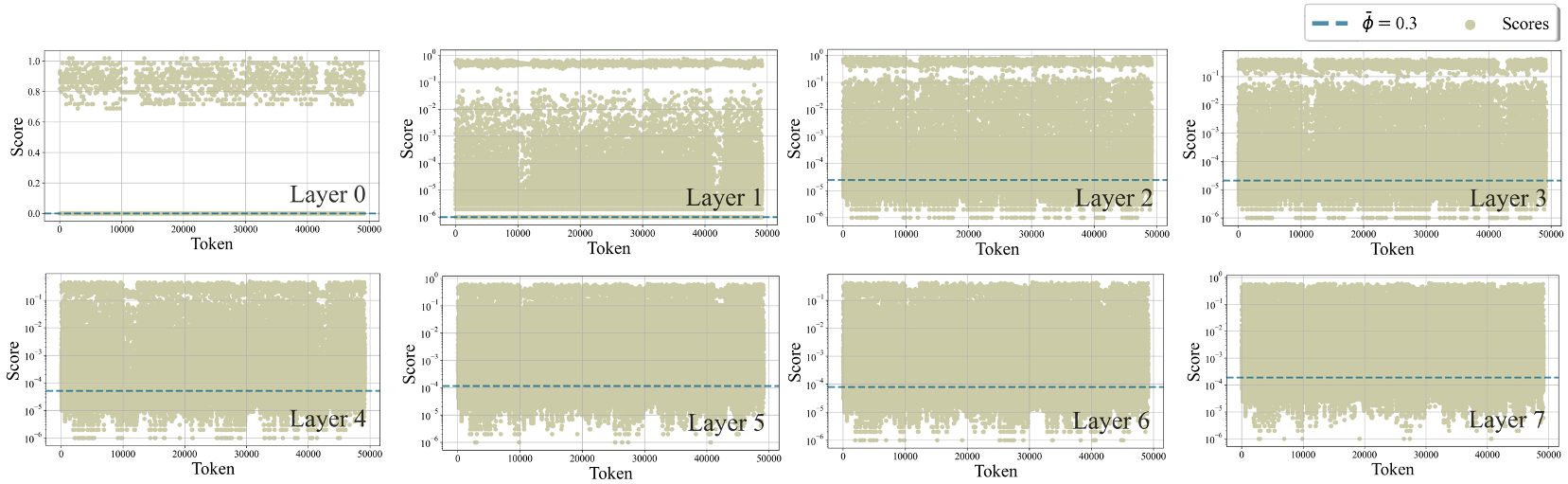}
    \caption{
    Per-token temporal importance scores of query vectors across self-attention layers.
    For token $i$ at time $t$, the score measures representational change between $q_t^i$ and $q_{t-1}^i$; $0$ indicates no change.
    Tokens whose scores are below the threshold (dashed line) are reused, while those above are refreshed.
    Early layers exhibit a bi-modal distribution with a strong zero-mode from masked tokens, whereas deeper layers show intertwined scores as contextual mixing increases.
    Because early layers devote much of their representation to static masked tokens, they can be thresholded aggressively with minimal accuracy loss.}

    \label{fig:per-layer-importance}
\end{figure*}
% Do we want to use "sparsity" allocation? or "reuse rate" allocation?
\paragraph{Layer-wise reuse allocation}
Figure~\ref{fig:sparsity_to_threshold} illustrates how the reuse rate hyperparameter $\phi$ determines the reuse threshold $\tau_{\ell}(\phi)$ assigned to each self-attention layer~$\ell$. 
As $\phi$ increases, thresholds become more permissive, leading to higher reuse in early layers and more selective reuse in later ones. 
This adaptive allocation reflects the observation that shallow layers exhibit stronger token-level redundancy and can tolerate more aggressive activation reuse without affecting output quality. 
Conversely, deeper layers encode higher-level semantic information and therefore require stricter thresholds to avoid propagating stale representations. 
The resulting per-layer reuse profile demonstrates that \ours\ effectively balances computational savings with representational fidelity by distributing reuse rate according to each layer’s redundancy characteristics.

\section{Analysis and Ablations}

\paragraph{Temporal importance scores across self-attention layers}

\autoref{fig:per-layer-importance} shows per-token query drift across the first nine self-attention layers, measuring the change between $q_t^i$ and $q_{t-1}^i$ (0 indicates no change). 
The dashed blue line denotes the threshold induced by $\phi{=}0.3$: tokens below are reused, while those above are refreshed.

Early layers exhibit a clear bimodal distribution, with many tokens concentrated at 0 (masked, unchanged) and a smaller set of active tokens. 
Deeper layers show more intertwined distributions as contextual interactions increase.

\textbf{This implies that early layers can be aggressively reused} with minimal accuracy loss, while deeper layers require more selective refresh due to stronger cross-token dependencies.

\paragraph{Per-layer redundancy structure}
Figure~\ref{fig:per-layer} shows token-wise similarity across timesteps in each self-attention layer, where $(i,j)$ denotes cosine similarity between keys at timesteps $i$ and $j$.
Shallow layers exhibit strong block-diagonal structure, indicating high local redundancy and suitability for aggressive reuse, while deeper layers show more diffuse patterns due to increased semantic integration.
This motivates our layer-adaptive reuse scheme, which applies higher reuse in early layers and more selective refresh in later layers, preserving quality while reducing computation.

\begin{figure}
\centering
\includegraphics[width=0.5\linewidth]{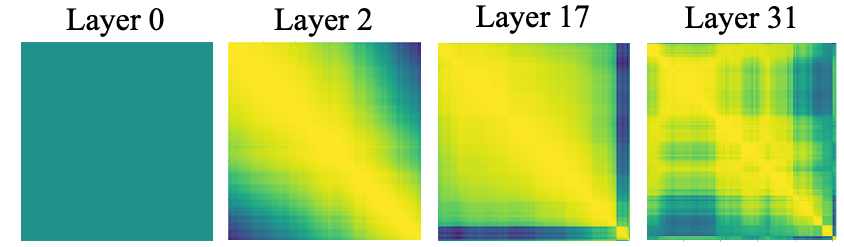}
\caption{
Per-layer similarity heatmap showing distinct reuse patterns across layers. The $(i,j)$ entry denotes similarity between keys at timesteps $i$ and $j$.
}
\label{fig:per-layer}
\end{figure}

\section{Conclusion}
We introduced \ours, a framework for token-wise activation reuse in diffusion language models (dLLMs) that reduces the cost of bi-directional self-attention.
By exploiting token-level redundancy, \kours\ reuses KV states and \oours\ reuses output activations without retraining or weight changes.
\oours\ achieves up to $1.20\times$ layer-level speedup while maintaining near-baseline accuracy on reasoning and code-generation tasks.

\section{LLM Usage}
We used large language models (LLMs) for assistance with code implementation, generating tables and figures, and improving the clarity and conciseness of the manuscript.
All outputs were reviewed and validated by the authors.

\section{Acknowledgments}
This work was supported in part by NSF CCF Grant No. 2107085, iMAGiNE - the Intelligent
Machine Engineering Consortium at UT Austin, and UT Cockrell School of Engineering Doctoral Fellowships.

\bibliography{colm2026_conference}
\bibliographystyle{colm2026_conference}

\appendix

\section{Experimental Setup Details}
\label{app:setup}

\paragraph{Benchmarks and models.}
We evaluate on GSM8K (5-shot arithmetic reasoning)~\citep{cobbe2021training},
HumanEval (Python synthesis)~\citep{chen2021evaluating},
Minerva Math (4-shot symbolic reasoning)~\citep{lewkowycz2022solving}, 
and MBPP (3-shot program completion)~\citep{austin2021program}.
Experiments are conducted on Dream-7B~\citep{ye2025dream}, LLaDA-8B~\citep{nie2025large}, and LLaDA-1.5-8B~\citep{zhu2025llada}.

\paragraph{Metrics.}
We report the percentage of reused KV or output activations and Flexible Extract accuracy, which treats semantically equivalent or format-invariant responses as correct.

\paragraph{Implementation details.}
Unless specified otherwise, we use batch size 1, block size 32, and generation length 256, with $\bar{\phi}=0.3$.
Reuse rates are computed dynamically at runtime and averaged across tokens.

\paragraph{Latency setup.}
Latency is measured on a single NVIDIA RTX A6000 GPU using BF16 precision, with batch size 32, block size 32, and sequence length 256.
Self-attention is implemented via \texttt{torch.nn.scaled\_dot\_product\_attention}.

\section{Reuse--latency trade-off}
\label{app:reuse_latency}

Table~\ref{tab:reuse_latency} reports an ablation on block size and reuse for LLaDA-8B, measuring both quality and runtime.
At $B{=}32$, \oours\ reuses 48.1\% of activations and reduces per-layer latency from 1.05s to 0.76s ($1.37\times$) with only a 1.2-point drop in Pass@1.
However, end-to-end throughput remains similar, indicating a memory-bound regime.
At larger sequence lengths, increased arithmetic intensity is expected to shift toward compute-bound behavior, where reuse yields larger gains.
Overall, selective activation reuse provides consistent wall-clock savings with minimal quality impact.

\begin{table}[h]
\centering

\caption{
Ablation of reuse and block size on performance and latency.
"Reuse" is activation reuse percentage, "Lat." is per-layer latency, and "Tok/s" is end-to-end throughput on LLaDA.
}
\label{tab:reuse_latency}
\begin{tabular}{lcccc}
\toprule
\textbf{Setting} & \textbf{GFLOPs} & \textbf{Reuse} & \textbf{Lat. (s)} & \textbf{Tok/s} \\
\midrule
\multicolumn{5}{c}{\textit{$B{=}32$}} \\
\midrule
Baseline & 11.6 & -- & \textbf{1.05} & \textbf{77.4} \\
\oours & \textbf{7.7} & 48.1\% & 0.76 & 73.2 \\
\midrule
\multicolumn{5}{c}{\textit{$B{=}128$}} \\
\midrule
Baseline & 47.5 & -- & \textbf{0.97} & 70.5 \\
\oours & \textbf{22.4} & 18.5\% & 0.86 & \textbf{71.4} \\
\bottomrule
\end{tabular}
\end{table}

\section{Effect of reuse rate allocation}
\label{app:phi}

\begin{table}[h]
\vspace{0pt}
\centering
\setlength{\tabcolsep}{8pt}
\captionof{table}{Ablation on reuse allocation parameter $\phi$ \textit{vs.} GSM8k (5-shot) Accuracy on LLaDA-8B.}
\label{tab:tau_ablation}
\begin{tabular}{lc}
\toprule
\textbf{$\phi$} & \textbf{GSM8k (5-shot)} \\
\midrule
 0.1 & 74.8\% \\
 0.2 & 74.0\% \\
0.3 & 76.6\% \\
0.5 & \textbf{78.2\%} \\
\bottomrule
\end{tabular}
\end{table}

Table~\ref{tab:tau_ablation} and Figure~\ref{fig:sparsity_to_threshold} show the effect of the reuse allocation parameter $\phi$, which controls layerwise reuse thresholds.
Higher $\phi$ increases reuse, improving GSM8K accuracy up to $\phi{=}0.5$ (78.2\%).
This suggests moderate reuse effectively exploits redundancy, while overly conservative settings underutilize it.
Thus, $\phi$ provides a simple and robust knob to balance reuse and accuracy.

\begin{figure}[h]
\centering
\includegraphics[width=0.7\linewidth]{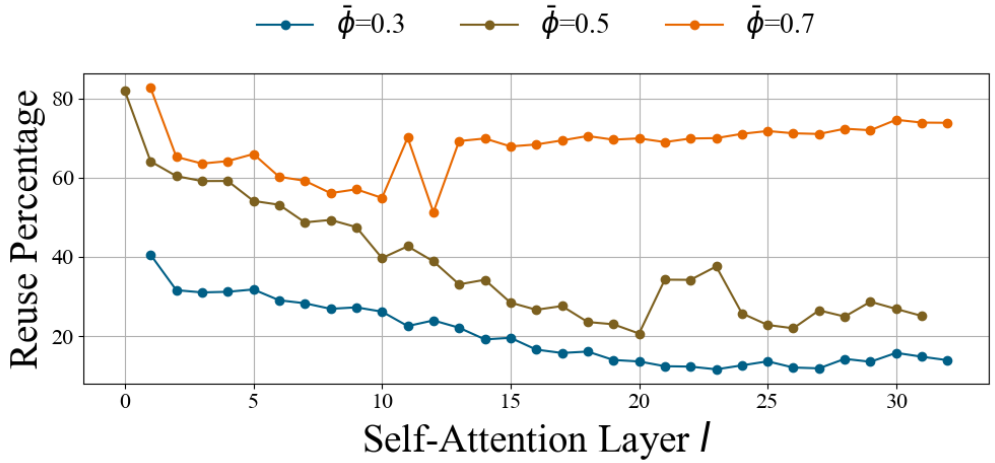}
\caption{
Layerwise thresholds induced by $\phi$.
Higher $\phi$ increases reuse, with average reuse rates of 20.2\%, 38.3\%, and 67.7\% for $\phi = 0.3, 0.5, 0.7$.
}
\label{fig:sparsity_to_threshold}
\end{figure}

\section{Full latency results across batch sizes}
\label{app:full_latency}

Table~\ref{tab:dare_latency_full} reports end-to-end performance across batch sizes.
While \ours improves efficiency in both regimes, gains are more pronounced at smaller batch sizes, where computation is less saturated.
At batch size 1, we observe up to 7.7$\times$ throughput improvements, while at batch size 16, gains remain strong (up to 5.8$\times$) but are reduced due to higher hardware utilization.

This trend reflects the transition from a more compute-bound regime at low batch sizes to a partially memory-bound regime at higher batch sizes.
Nevertheless, activation reuse consistently reduces redundant computation and yields meaningful speedups across all settings.

\begin{table*}[t]
\centering
\caption{
\textbf{End-to-end performance with \ours\ across batch sizes.}
}
\label{tab:dare_latency_full}
\begin{tabular}{@{}cc|ccc|ccc|ccc@{}}
\toprule
\textbf{Batch} & \textbf{Seq} &
\multicolumn{3}{c}{\textbf{Tok/s}} &
\multicolumn{3}{c}{\textbf{TTLT}} &
\multicolumn{3}{c}{\textbf{TPOT}} \\
\cmidrule(lr){3-5} \cmidrule(lr){6-8} \cmidrule(lr){9-11}
& &
Base & +\ours & ($\times$) &
Base & +\ours & (\%) &
Base & +\ours & (\%) \\
\midrule

\multirow{4}{*}{1}
& 256 & 17.7 & \textbf{55.5} & 3.1$\times$ & 14482 & \textbf{4611} & 68\% & 56.6 & \textbf{18.0} & 68\% \\
& 512 & 17.7 & \textbf{136.6} & 7.7$\times$ & 28995 & \textbf{3747} & 87\% & 56.6 & \textbf{7.3} & 87\% \\
& 1024 & 20.8 & \textbf{97.8} & 4.7$\times$ & 49283 & \textbf{10471} & 79\% & 48.1 & \textbf{10.2} & 79\% \\
& 2048 & 57.7 & \textbf{59.2} & 1.0$\times$ & 35485 & \textbf{34619} & 2\% & 17.3 & \textbf{16.9} & 2\% \\
\midrule

\multirow{4}{*}{16}
& 256 & 17.7 & \textbf{27.0} & 1.5$\times$ & 14490 & \textbf{9497} & 34\% & 56.6 & \textbf{37.1} & 34\% \\
& 512 & 17.7 & \textbf{103.1} & 5.8$\times$ & 28975 & \textbf{4964} & 83\% & 56.6 & \textbf{9.7} & 83\% \\
& 1024 & 16.1 & \textbf{97.8} & 6.1$\times$ & 63764 & \textbf{10469} & 84\% & 62.3 & \textbf{10.2} & 84\% \\
& 2048 & 58.5 & \textbf{59.4} & 1.0$\times$ & 35016 & \textbf{34460} & 2\% & 17.1 & \textbf{16.8} & 2\% \\
\bottomrule
\end{tabular}
\end{table*}

\section{Theoretical Analysis of \ours}

\subsection{Notations and Update Rule of \ours}

We consider generating tokens by a single-layer transformer with single-head self-attention. The weights of query, key, value, and output are expressed as $W_Q, W_K, W_V, W_O\in\R^{d\times d}$.  In addition, we denote the up and down projection matrices in the Multi-Layer Perceptron (MLP) by $W_U\in \R^{d\times d_\text{int}}$, $W_D \in \R^{d_\text{int} \times d}$. We write $\mathbf{x}\in\R^{1\times d}$ to denote a token or its layer-normalized input embedding interchangeably. Let $E\in\R^{n_\text{vocab}\times d}$ be the embedding matrix for the final output layer.

Given a reuse threshold $\tau$, the $t$-th decoding step of \kours can be written as:
\begin{align*}
&\text{Retrieve } \widehat{Q}^{t-1}, \widehat{K}^{t-1}, \widehat{V}^{t-1}\text{ and } \widehat{Q}^t = \widehat{X}^t W_Q\in  \R^{B\times d}\\
& \text{Collect the reuse indices }\mathcal{R}^t = \{i\mid \langle \widehat{\mathbf{q}}^t_i, \widehat{\mathbf{q}}^{t-1}_i\rangle\geq 1-\tau\}, \quad \widehat{\mathbf{k}}_i^t = \begin{cases}
\widehat{\mathbf{k}}^{t-1}_i, & i \in \mathcal{R}^t\\
\mathbf{x}_i^t W_K, & i \notin \mathcal{R}^t
\end{cases}, \quad \widehat{\mathbf{v}}^t_i = \begin{cases}
\widehat{\mathbf{v}}^{t-1}_i, & i \in \mathcal{R}^t\\
\mathbf{x}_i^t W_V, & i \notin \mathcal{R}^t
\end{cases},\\
& \widehat{H}^t = \sigma\left(\left(\text{Softmax}\left(\frac{\widehat{Q}^t [\widehat{K}^t]^\top}{\sqrt{d}}\right)\widehat{V}^t W_O\right)W_U\right) W_D\in\R^{B\times d},\quad  \widehat{p}_i^t(\cdot\mid \widehat{X}^t) = \text{Softmax}(\widehat{\mathbf{h}}_t^i E^\top),
\end{align*}
where $\sigma$ is a nonlinear activation, $\widehat{h}^t_i$ is the $i$-th row of $\widehat{H}^t$, and $\widehat{p}_i^t(\cdot\mid \widehat{X}^t) $ is a probability distribution over the vocabulary conditioned on the input sequence $\widehat{X}^t$. We then sample the next token $\widehat{j}_i^t\sim \widehat{p}_i^t(\cdot\mid \widehat{X}^t)$ and map it into layer-normalized input embedding $\widehat{\mathbf{x}}_i^{t+1} = f_E(\widehat{j}_i^t)$, where $f_E:\{1,\dotsc,n_\text{vocab}\}\rightarrow \R^d$. 

The $t$-th decoding step of \oours can be written as:
\begin{align*}
&\text{Retrieve } \widehat{Q}^{t-1}, \widehat{O}^{t-1}\text{ and } \widehat{Q}^t = \widehat{X}^t W_Q,\widehat{K}^t = \widehat{X}^t W_K, \widehat{V}^t = \widehat{X}^t W_V\in  \R^{B\times d}\\
& \text{Collect the reuse indices }\mathcal{R}^t = \{i\mid \langle \widehat{\mathbf{q}}^t_i, \widehat{\mathbf{q}}^{t-1}_i\rangle\geq 1-\tau\}, \quad \widehat{\mathbf{o}}_i^t = \begin{cases}
\widehat{\mathbf{o}}^{t-1}_i, & i \in \mathcal{R}^t\\
\text{Softmax}\left(\frac{\widehat{\mathbf{q}}_i^t [\widehat{K}^t]^\top}{\sqrt{d}}\right)\widehat{V}^t W_O, & i \notin \mathcal{R}^t
\end{cases}, \\
& \widehat{H}^t = \sigma(\widehat{O}^t W_U) W_D\in\R^{B\times d},\quad  \widehat{p}_i^t(\cdot\mid \widehat{X}^t) = \text{Softmax}(\widehat{\mathbf{h}}_t^i E^\top).
\end{align*}

\subsection{The Full dLLM Reference Sequence and the Expected Cumulative KV Reuse Error}\label{sec:full_seq}

The forward pass of full dLLM (\textit{i.e.}, without KV reuse) can be written as
\begin{align*}
& Q^t = X^t W_Q\in  \R^{B \times d},\quad K^t = X^t W_K\in \R^{B\times d}, \quad V^t = X^t W_V\in \R^{B\times d},\\
& H^t = \sigma\left(\left(\text{Softmax}\left(\frac{Q^t [K^t]^\top}{\sqrt{d}}\right)V^t W_O\right)W_U\right) W_D\in\R^{B\times d},\quad p_i^t(\cdot\mid X^t) = \text{Softmax}(\mathbf{h}^t_i E^\top).
\end{align*}
One may construct the reference sequence $(X^t)_{t\geq 0}$ of the full dLLM as $j_i^t\sim p_i^t(\cdot\mid X^t)$,  $\mathbf{x}_i^{t+1} = f_E(j_i^t)$, and upper bound:
\begin{align*}
\E\left[\sum_{i=1}^B\|\mathbf{x}_i^t - \widehat{\mathbf{x}}_i^t\|_2\right].
\end{align*}
However, this does not clearly isolate the KV reuse error: Even if $X^t = \widehat{X}^t$ and there is no KV reuse in the $t$-th step (\textit{i.e.}, $p_i^t(\cdot\mid X^t) = \widehat{p}_i^t(\cdot\mid \widehat{X}^t)$), we still have $\E\left[\sum_{i=1}^B\|\mathbf{x}_i^t - \widehat{\mathbf{x}}_i^t\|_2\right]>0$ due to the independent sampling of $j_i^t$ and $\widehat{j}_i^t$. Consequently, the term above is dominated by the unknown sampling variance accumulated across all steps.

Instead, we consider the joint distribution\footnote{Known as \emph{maximal coupling}.} $\pi_i^t$ of $(j_i^t,\widehat{j}_i^t)$ with marginal distributions $j_i^t \sim p_i^t(\cdot\mid X^t)$ and $\widehat{j}_i^t \sim \widehat{p}_i^t(\cdot\mid \widehat{X}^t)$. Then, we construct the reference sequence $(X^t)_{t\geq 0}$ generated by full dLLM as $\mathbf{x}_i^{t+1} = f_E(j_i^t)$. To be specific, the diagonal masses of $\pi_i^t$ are set to be $$\pi_i^t(j_i^t = j, \widehat{j}_i^t = j) = \min\left\{p_i^t(j_i^t = j\mid X^t), \widehat{p}_i^t(\widehat{j}_i^t = j\mid \widehat{X}^t)\right\},\quad \forall j\in\{1,\dotsc,n_\text{vocab}\}.$$
The off-diagonal masses only need to satisfy the marginal distribution constraints:
\begin{align*}
& \sum_{j', j'\neq j}\pi_i^t(j_i^t = j, \widehat{j}_i^t = j') = p_i^t(j_i^t = j\mid X^t) - \pi_i^t(j_i^t = j, \widehat{j}_i^t = j),\\
& \sum_{j, j\neq j'}\pi_i^t(j_i^t = j, \widehat{j}_i^t = j') = \widehat{p}_i^t(\widehat{j}_i^t = j'\mid \widehat{X}^t) - \pi_i^t(j_i^t = j', \widehat{j}_i^t = j').
\end{align*}
This joint distribution ensures that when $p_i^t(\cdot\mid X^t) = \widehat{p}_i^t(\cdot\mid \widehat{X}^t)$, it holds that $j_i^t = \widehat{j}_i^t$ almost surely. To this end, we can isolate the per-step KV reuse error, \textit{i.e.}, $\E\left[\sum_{i=1}^B\|\mathbf{x}_i^t - \widehat{\mathbf{x}}_i^t\|_2\right]=0$ when $X^t = \widehat{X}^t$ and there is no KV reuse in the $t$-th step, \textit{i.e.}, $p_i^t(\cdot\mid X^t) = \widehat{p}_i^t(\cdot\mid X^t)$. We formally state this in the following proposition.
\begin{proposition}\label{prop:isolate}
Let $(\widehat{X}^t)_{t\geq 0}$ be a real sequence generated by \ours. Besides, $(X^t)_{t\geq 0}$ is the reference sequence of full dLLM as constructed above. If $X^t = \widehat{X}^t$ and there is no KV reuse (\textit{i.e.}, $p_i(\cdot\mid X_t) = \widehat{p}_i(\cdot\mid X_t)$ for all $i$), then we have 
$$ \|\mathbf{x}_i^{t+1} - \widehat{\mathbf{x}}_i^{t+1}\|_2 = 0 \text{  almost surely}\quad \text{and}\quad \E\left[\sum_{i=1}^B\|\mathbf{x}_i^{t+1} - \widehat{\mathbf{x}}_i^{t+1}\|_2\right]=0. $$
\end{proposition}
\begin{proof}
Let's focus on a single term $\E\left[\|\mathbf{x}_i^t - \widehat{\mathbf{x}}_i^t\|_2\right]$ for an arbitrary index $i$. For any outcome $k$, the joint probability is:
$$ \pi_i^t(j, j) = \min\{p_i^t(j\mid X^t), \widehat{p}_i^t(j\mid \widehat{X}^t)\} = p_i^t(j\mid X^t).$$
The total probability on the diagonal is $\sum_j \pi_i^t(j, j)  = \sum_j p_i^t(j\mid X^t) = 1$. This implies that under this coupling, the event $j_i = \widehat{j}_i$ occurs with probability 1. Since $\mathbf{x}_i^{t+1} = f_E(j_i^t)$ and $\widehat{\mathbf{x}}_i^{t+1} = f_E(\widehat{j}_i^t)$, the condition $j_i^t = \widehat{j}_i^t$ implies that $\mathbf{x}_i^{t+1} = \widehat{\mathbf{x}}_i^{t+1}$. Thus, $\|\mathbf{x}_i^{t+1} - \widehat{\mathbf{x}}_i^{t+1}\|_2=0$ with probability 1.
\end{proof}
\begin{remark}
In practice, the property above can be achieved by setting the same random seed, which forces the underlying sequence of pseudo-random numbers to be identical.
\end{remark}

Another nice property of $\pi_i^t$ is that the expected distance between embeddings $\mathbf{x}_i^{t+1}$ and  $\widehat{\mathbf{x}}_i^{t+1}$ can be controlled by the total variation distance between the marginal distributions $p_i(\cdot\mid X^t)$ and $\widehat{p}_i(\cdot\mid \widehat{X}^t)$.

\begin{lemma}\label{lem:dist_to_emb}

For the $\pi_i^t$ described above and $(j_i^t,\widehat{j}_i^t)\sim \pi_i^t$, the expected distance between embeddings $\mathbf{x}_i^{t+1} = f_E(j_i^t)$ and $\widehat{\mathbf{x}}_i^{t+1}=f_E(\widehat{j}_i^t)$ is bounded by the total variation distance between the marginal distributions $p_i^t(\cdot\mid X^t)$ and $\widehat{p}_i^t(\cdot\mid \widehat{X}^t)$:
$$ \mathbb{E} \|\mathbf{x}_i^{t+1} - \widehat{\mathbf{x}}_i^{t+1}\|_2 = \mathbb{E}\|f_E(j_i^t) - f_E(\widehat{j}_i^t)\|_2 \leq  \|p_i^t(\cdot\mid X^t) - \widehat{p}_i^t(\cdot\mid \widehat{X}^t)\|_1,\quad \forall i\in\{1,\dotsc,B\}.$$
\end{lemma}

\begin{proof}
Recall that the diagonal mass is set to $$\pi_i^t(j,j) = \min\{p_i^t(j\mid X^t), \widehat{p}_i^t(j\mid \widehat{X}^t)\},\quad \forall i \in \{1,\dotsc,n_\text{vocab}\}.$$
The remaining off-diagonal masses satisfy:
$$ \sum_{j \neq j'} \pi_i^t(j,j') = 1 - \sum_j \pi_i^t(j,j) = 1 - \sum_j \min\{p_i^t(j\mid X^t), \widehat{p}_i^t(j\mid \widehat{X}^t)\}.$$
We relate this off-diagonal mass to the Total Variation (TV) distance, $d_\mathrm{TV}(p_i^t(\cdot\mid X^t), \widehat{p}_i^t(\cdot\mid \widehat{X}^t)) = \frac{1}{2}\|p_i^t(\cdot\mid X^t) - \widehat{p}_i^t(\cdot\mid \widehat{X}^t)\|_1$. Using the identity $\sum_j \min\{p_i^t(j\mid X^t), \widehat{p}_i^t(j\mid \widehat{X}^t)\} = 1 - d_\mathrm{TV}(p_i^t(\cdot\mid X^t), \widehat{p}_i^t(\cdot\mid \widehat{X}^t))$, the total off-diagonal mass is precisely the TV distance.
$$ \sum_{j \neq j'} \pi_i^t(j,j') = 1 - (1 - d_\mathrm{TV}(p_i^t(\cdot\mid X^t), \widehat{p}_i^t(\cdot\mid \widehat{X}^t))) = d_\mathrm{TV}(p_i^t(\cdot\mid X^t), \widehat{p}_i^t(\cdot\mid \widehat{X}^t)).$$
Now, we can write the expectation and split the sum into diagonal and off-diagonal components:
\begin{align*}
\mathbb{E}[\|f_E(j_i^t) - f_E(\widehat{j}_i^t)\|_2] &= \sum_{j,j'} \pi_i^t(j,j') \|f_E(j) - f_E(j')\|_2 \\
&= \sum_{j} \pi_i^t(j,j) \|f_E(j) - f_E(j)\|_2 + \sum_{j \neq j'} \pi_i^t(j,j') \|f_E(j) - f_E(j')\|_2.
\end{align*}
The first term above is zero. For the second term, we have 
 $\|f_E(i) - f_E(j)\|_2 \leq \|f_E(i)\|_2 + \|f_E(j)\|_2 \leq 2R $.
Substituting this bound into the expression for the expectation:
\begin{align*}
\mathbb{E}[\|f_E(j_i^t) - f_E(\widehat{j}_i^t)\|_2] &= \sum_{j \neq j'} \pi_i^t(j,j') \|f_E(j) - f_E(j')\|_2  \leq 2R \sum_{i \neq j} \pi_i^t(j,j').
\end{align*}
Using our result that the sum of the off-diagonal probabilities is the TV distance:
\begin{align*}
\mathbb{E}[\|f_E(j_i^t) - f_E(\widehat{j}_i^t)\|_2] &\leq 2R \cdot d_\mathrm{TV}(p_i^t(\cdot\mid X^t), \widehat{p}_i^t(\cdot\mid \widehat{X}^t)) = R \|p_i^t(\cdot\mid X^t) - \widehat{p}_i^t(\cdot\mid \widehat{X}^t)\|_1.
\end{align*}
\end{proof}

\subsection{Proof of Theorem~\ref{thm:main}}\label{sec:thm_proof}

For \kours, Theorem~\ref{thm:main} follows from (\ref{eq:decomp}), Lemma~\ref{lem:lip}, and Lemma~\ref{lem:reuse_err}. For \oours, we replace Lemma~\ref{lem:reuse_err} by the analysis in Appendix~\ref{sec:bound_oreuse}.

\subsubsection{Error Decomposition}

We use $\E_t$ to represent the expectation conditioned on randomness up to (and including) step $t$. Due to Lemma~\ref{lem:dist_to_emb}, 
\begin{align}\nonumber
& \E_t \sum_{i=1}^B \|\mathbf{x}_i^{t+1} - \widehat{\mathbf{x}}_i^{t+1}\|_2  \leq   \sum_{i=1}^B \|p_i^t(\cdot\mid X^t) - \widehat{p}_i^t(\cdot\mid\widehat{X}^t)\|_1\\\label{eq:decomp}
& \leq \underbrace{\sum_{i=1}^B \|p_i^t(\cdot\mid X^t) - p_i^t(\cdot\mid\widehat{X}^t)\|_1}_{\text{Error Propagated from Previous Step}} + \underbrace{\sum_{i=1}^B \|p_i^t(\cdot\mid\widehat{X}^t) - \widehat{p}_i^t(\cdot\mid\widehat{X}^t)\|_1}_{\text{Reuse Error in Current Step}}.
\end{align}
We handle the first term in (\ref{eq:decomp}) by the Lipschitz continuity of full dLLM w.r.t. the input sequence. The second term in (\ref{eq:decomp}) is bounded in Appendix~\ref{sec:per_step} (KV reuse) or Appendix~\ref{sec:bound_oreuse} (O reuse).

\subsubsection{Lipschitz Continuity w.r.t. the Input}\label{sec:lip} 

The next lemma provides an explicit upper bound on the Lipschitz continuity constant of full dLLM w.r.t. its input sequence.
\begin{lemma}\label{lem:lip}
Suppose that the nonlinear activation is $G_\sigma$-Lipschitz continuous (\textit{e.g.}, ReLU has $G_\sigma = 1$). Then, we have $$\sum_{i=1}^B \|p_i^t(\cdot\mid X^t) - p_i^t(\cdot\mid\widehat{X}^t)\|_1 \leq L \sum_{i=1}^B \|\mathbf{x}_i^t - \widehat{\mathbf{x}}_i^t\|_2$$ with
\begin{align*}
G \leq  \|E\|_{2\to1}\,
\|W_{D}\|_{2}\,
G_{\sigma}\,
\|W_{U}\|_{2}\,
\|W_{O}\|_{2}\,
\|W_{V}\|_{2}\,
B \left(
  \frac{2R^{2}\,\|W_{Q}\|_{2}\,\|W_{K}\|_{2}}{\sqrt{d}} + 1
\right).
\end{align*}
\end{lemma}
\begin{proof}
Let $\mathbf{l}_i^t = \mathbf{h}_i^t E^\top$ be the logits for the $i$-th token, $\mathbf{h}_i^t$ the $i$-th row of $H^t$, $\mathbf{o}_i^t$ the $i$-th row of $O^t = A^t V^t W_O$, and $\mathbf{a}_i^t$ the $i$-th row of $A^t = \text{Softmax}(S^t)$. We compare this to the perturbed computation denoted with $\widehat{\cdot}$, \textit{e.g.}, $\widehat{\mathbf{l}}_i^t = \widehat{\mathbf{h}}_i^t E^\top$.

\begin{align*}
\sum_{i=1}^B \|p_i^t(\cdot \mid X^t) - p_i^t(\cdot \mid \widehat{X}^t)\|_1 & \le \sum_{i=1}^B \|\mathbf{l}_i^t - \widehat{\mathbf{l}}_i^t\|_1 \quad (\text{Softmax is 1-Lipschitz } L_1 \to L_1) \\
& = \sum_{i=1}^B \|(\mathbf{h}_i^t - \widehat{\mathbf{h}}_i^t) E^\top\|_1 \le \|E\|_{2\to 1} \sum_{i=1}^B \|\mathbf{h}_i^t - \widehat{\mathbf{h}}_i^t\|_2 \\
& \le \|E\|_{2\to 1} \sum_{i=1}^B \|(\sigma(\mathbf{o}_i^t W_U) - \sigma(\widehat{\mathbf{o}}_i^t W_U)) W_D\|_2 \\
& \le \|E\|_{2\to 1} \|W_D\|_2 \sum_{i=1}^B \|\sigma(\mathbf{o}_i^t W_U) - \sigma(\widehat{\mathbf{o}}_i^t W_U)\|_2 \\
& \le \|E\|_{2\to 1} \|W_D\|_2 G_\sigma \sum_{i=1}^B \|(\mathbf{o}_i^t - \widehat{\mathbf{o}}_i^t) W_U\|_2 \\
& \le \|E\|_{2\to 1} \|W_D\|_2 G_\sigma \|W_U\|_2 \sum_{i=1}^B \|\mathbf{o}_i^t - \widehat{\mathbf{o}}_i^t\|_2 \\
& \le C_1 \sum_{i=1}^B \|(\mathbf{a}_i^t V^t - \widehat{\mathbf{a}}_i^t \widehat{V}^t) W_O\|_2 \quad (\text{where } C_1 = \|E\|_{2\to 1} \|W_D\|_2 G_\sigma \|W_U\|_2) \\
& \le C_1 \|W_O\|_2 \sum_{i=1}^B \left( \|\mathbf{a}_i^t(V^t - \widehat{V}^t)\|_2 + \|(\mathbf{a}_i^t - \widehat{\mathbf{a}}_i^t)\widehat{V}^t\|_2 \right).
\end{align*}
We bound the two terms in the sum:
\begin{enumerate}
    \item $\sum_i \|\mathbf{a}_i^t(V^t - \widehat{V}^t)\|_2 \le \sum_i \sum_j a_{ij}^t \|\mathbf{v}_j^t - \widehat{\mathbf{v}}_j^t\|_2 = \sum_j (\sum_i a_{ij}^t) \|\mathbf{v}_j^t - \widehat{\mathbf{v}}_j^t\|_2 \\ \le B \sum_j \|\mathbf{v}_j^t - \widehat{\mathbf{v}}_j^t\|_2 \le B \|W_V\|_2 \sum_j \|\mathbf{x}_j^t - \widehat{\mathbf{x}}_j^t\|_2$;
    
    \item $\sum_i \|(\mathbf{a}_i^t - \widehat{\mathbf{a}}_i^t)\widehat{V}^t\|_2 \le \sum_i \|\mathbf{a}_i^t - \widehat{\mathbf{a}}_i^t\|_1 \max_j \|\widehat{\mathbf{v}}_j^t\|_2 \le R \|W_V\|_2 \sum_i \|\mathbf{a}_i^t - \widehat{\mathbf{a}}_i^t\|_1$.
\end{enumerate}
Next, we bound $\sum_i \|\mathbf{a}_i^t - \widehat{\mathbf{a}}_i^t\|_1 \le \sum_i \|\mathbf{s}_i^t - \widehat{\mathbf{s}}_i^t\|_1 = \|S^t - \widehat{S}^t\|_1 = \sum_{i,j} |s_{ij}^t - \widehat{s}_{ij}^t|$.
\begin{align*}
|s_{ij}^t - \widehat{s}_{ij}^t| & = \frac{1}{\sqrt{d}} |\mathbf{q}_i^t (\mathbf{k}_j^t)^\top - \widehat{\mathbf{q}}_i^t (\widehat{\mathbf{k}}_j^t)^\top| \le \frac{1}{\sqrt{d}} \left( \|\mathbf{q}_i^t\|_2 \|\mathbf{k}_j^t - \widehat{\mathbf{k}}_j^t\|_2 + \|\mathbf{q}_i^t - \widehat{\mathbf{q}}_i^t\|_2 \|\widehat{\mathbf{k}}_j^t\|_2 \right) \\
& \le \frac{1}{\sqrt{d}} \left( (R \|W_Q\|_2) (\|W_K\|_2 \|\mathbf{x}_j^t - \widehat{\mathbf{x}}_j^t\|_2) + (\|W_Q\|_2 \|\mathbf{x}_i^t - \widehat{\mathbf{x}}_i^t\|_2) (R \|W_K\|_2) \right) \\
& = \frac{R \|W_Q\|_2 \|W_K\|_2}{\sqrt{d}} \left( \|\mathbf{x}_i^t - \widehat{\mathbf{x}}_i^t\|_2 + \|\mathbf{x}_j^t - \widehat{\mathbf{x}}_j^t\|_2 \right).
\end{align*}
Summing over $i, j$:
\begin{align*}
\sum_i \|\mathbf{a}_i^t - \widehat{\mathbf{a}}_i^t\|_1 & \le \sum_{i,j} |s_{ij}^t - \widehat{s}_{ij}^t| \le \frac{R \|W_Q\|_2 \|W_K\|_2}{\sqrt{d}} \sum_{i,j} \left( \|\mathbf{x}_i^t - \widehat{\mathbf{x}}_i^t\|_2 + \|\mathbf{x}_j^t - \widehat{\mathbf{x}}_j^t\|_2 \right) \\
& = \frac{2 B R \|W_Q\|_2 \|W_K\|_2}{\sqrt{d}} \sum_{i=1}^B \|\mathbf{x}_i^t - \widehat{\mathbf{x}}_i^t\|_2.
\end{align*}
Substitute this into Term 2:
$$\text{Term 2} \le (R \|W_V\|_2) \left( \frac{2 B R \|W_Q\|_2 \|W_K\|_2}{\sqrt{d}} \sum_i \|\mathbf{x}_i^t - \widehat{\mathbf{x}}_i^t\|_2 \right).$$
Combine Term 1 and the new bound for Term 2:
\begin{align*}
\sum_{i=1}^B \|\mathbf{o}_i^t - \widehat{\mathbf{o}}_i^t\|_2 & \le \|W_O\|_2 \left( B \|W_V\|_2 \sum_i \|\mathbf{x}_i^t - \widehat{\mathbf{x}}_i^t\|_2 + \frac{2 B R^2 \|W_V\|_2 \|W_Q\|_2 \|W_K\|_2}{\sqrt{d}} \sum_i \|\mathbf{x}_i^t - \widehat{\mathbf{x}}_i^t\|_2 \right) \\
& = \|W_O\|_2 \|W_V\|_2 B \left( 1 + \frac{2 R^2 \|W_Q\|_2 \|W_K\|_2}{\sqrt{d}} \right) \sum_i \|\mathbf{x}_i^t - \widehat{\mathbf{x}}_i^t\|_2.
\end{align*}
Finally, substituting this back into the first inequality chain:
\begin{align*}
& \sum_{i=1}^B \|p_i^t(\cdot \mid X^t) - p_i^t(\cdot \mid \widehat{X}^t)\|_1 \le C_1 \sum_{i=1}^B \|\mathbf{o}_i^t - \widehat{\mathbf{o}}_i^t\|_2 \\
& \le \|E\|_{2\to 1} \|W_D\|_2 G_\sigma \|W_U\|_2 \|W_O\|_2 \|W_V\|_2 B \left( \frac{2R^{2}\,\|W_{Q}\|_{2}\,\|W_{K}\|_{2}}{\sqrt{d}} + 1 \right) \sum_{i=1}^B \|\mathbf{x}_i^t - \widehat{\mathbf{x}}_i^t\|_2.
\end{align*}
\end{proof}

\subsubsection{Bounding the KV Reuse Error in Each Step}\label{sec:per_step}

Utilizing the notion of staleness vector $\Delta^t(\tau) = (\delta_1^t(\tau),\dotsc, \delta_B^t(\tau))^\top$ defined in Section~\ref{sec:theory}, we prove the following lemma. 
\begin{lemma}\label{lem:reuse_err}
Suppose that the nonlinear activation is $G_\sigma$-Lipschitz continuous. Thus, the KV reuse error in step $t$ can be upper bounded as:
\begin{align*}
& \sum_{i=1}^B \|p_i^t(\cdot\mid\widehat{X}^t) - \widehat{p}_i^t(\cdot\mid\widehat{X}^t)\|_1 \\
& \leq \sqrt{2\Tilde{\tau} }B \|E\|_{2\rightarrow\infty} \|W_D\|_2 G_\sigma \|W_U\|_2 \|W_O\|_2\left(\|W_V\|_2  + \frac{\|W_V\|_2 \|W_Q\|_2}{\sqrt{d}}\right) \|\Delta^t\|_2,
\end{align*}
where $\Delta^t := (\delta_{1,t},\dotsc,\delta_{B,t})^\top$ is the staleness vector of KV activations at step $t$, $\Tilde{\tau} := \frac{2\tau \cdot d \cdot\kappa(W_Q)^2 }{2+\tau (\kappa(W_Q)^2 - 1)}$, and $\kappa(W_Q):= \frac{\sigma_{\max}(W_Q)}{\sigma_{\min}(W_Q)}$.
\end{lemma}
For simplicity, we denote $C_W= \sqrt{2}B \|E\|_{2\rightarrow\infty} \|W_D\|_2 G_\sigma \|W_U\|_2 \|W_O\|_2\left(\|W_V\|_2  + \frac{\|W_V\|_2 \|W_Q\|_2}{\sqrt{d}}\right)$ such that
\begin{align*}
\sum_{i=1}^B \|p_i^t(\cdot\mid\widehat{X}^t) - \widehat{p}_i^t(\cdot\mid\widehat{X}^t)\|_1  \leq C_W\sqrt{\tau} \|\Delta^t\|_2.
\end{align*}
Intuitively, the lemma above implies that the per-step error $\|p_i^t(\cdot\mid\widehat{X}^t) - \widehat{p}_i^t(\cdot\mid\widehat{X}^t)\|_1 = 0$ if we do not reuse KV ($\tau = 0$).

First, we present a lemma on Lipschitz continuity of the softmax operation.
\begin{lemma}\label{lem:softmax_lip}
For vectors $z,z'$, we have $\|\mathrm{Softmax}(z) - \mathrm{Softmax}(z')\|_1\leq \|z-z'\|_\infty$.
\end{lemma}
\begin{proof}
Let $\sigma=\operatorname{softmax}$. By the mean-value theorem, 
\[
\sigma(z')-\sigma(z)=J(\xi)(z'-z)
\]
for some $\xi$ on the segment connecting $z$ to $z'$, where $J(\xi)=\operatorname{diag}(p)-pp^\top$ with $p=\sigma(\xi)$. Write $\mu=\sum_i p_i h_i$ such that
\[
\|J(\xi)h\|_1=\sum_i p_i|h_i-\mu|.
\]
If $\|h\|_\infty\le1$, then each $h_i\in[-1,1]$, so $\mu\in[-1,1]$ and hence $|h_i-\mu|\le2$. Thus the convex combination $\sum_i p_i|h_i-\mu|$ is maximized when all mass is on two points $-1$ and $1$, giving value $1$. Therefore $\|J(\xi)h\|_1\le\|h\|_\infty$ for all $h$. Taking $h=z'-z$ completes the proof.
\end{proof}

\begin{proof}[Proof of Lemma~\ref{lem:reuse_err}]
Due to the Lipschitz continuity of the softmax operator (Lemma~\ref{lem:softmax_lip}), for each $i\in[n]$ we have
\begin{align*}
& \|p_i^t(\cdot\mid \widehat{X}^t) -  \widehat{p}_i^t(\cdot\mid \widehat{X}^t) \|_1 = \|\text{Softmax}(h_i^t(\widehat{X}^t)E^\top) - \text{Softmax}(\widehat{h}_i^t(\widehat{X}^t)E^\top)\|_1\\
& \leq \|(h_i^t(\widehat{X}^t) - \widehat{h}_i^t(\widehat{X}^t))E^\top\|_\infty \leq \|E\|_{2\rightarrow\infty}\|h_i^t(\widehat{X}^t) - \widehat{h}_i^t(\widehat{X}^t)\|_2\\
& = \|E\|_{2\rightarrow\infty}\left\|\sigma\left(\left(\text{Softmax}\left(\frac{\widehat{\mathbf{q}}_i^t [\widehat{X}^t W_K]^\top}{\sqrt{d}}\right)(\widehat{X}^t W_V) W_O\right)W_U\right) W_D \right.\\
& \quad\quad\quad\quad \quad\quad \quad\quad\quad\quad   \left. - \sigma\left(\left(\text{Softmax}\left(\frac{\widehat{\mathbf{q}}_i^t [\widehat{K}^t]^\top}{\sqrt{d}}\right)\widehat{V}^t W_O\right)W_U\right) W_D\right\|_2
\end{align*}
Due to the $G_\sigma$-Lipschitz continuity of the nonlinear activation $\sigma$, we further have
\begin{align*}
& \|p_i^t(\cdot\mid \widehat{X}^t) -  \widehat{p}_i^t(\cdot\mid \widehat{X}^t) \|_1 \\
& \leq \|E\|_{2\rightarrow\infty} \|W_D\|_2 G_\sigma \|W_U\|_2 \|W_O\|_2 \left\|\text{Softmax}\left(\frac{\widehat{\mathbf{q}}_i^t [\widehat{X}^t W_K]^\top}{\sqrt{d}}\right)(\widehat{X}^t W_V) - \text{Softmax}\left(\frac{\widehat{\mathbf{q}}_i^t [\widehat{K}^t]^\top}{\sqrt{d}}\right)\widehat{V}^t\right\|_2\\
& \leq \|E\|_{2\rightarrow\infty} \|W_D\|_2 G_\sigma \|W_U\|_2 \|W_O\|_2\left(\left\|\text{Softmax}\left(\frac{\widehat{\mathbf{q}}_i^t [\widehat{X}^t W_K]^\top}{\sqrt{d}}\right)\left((\widehat{X}^t W_V) - \widehat{V}^t\right)\right\|_2\right.\\
& \quad\quad\quad\quad \quad\quad \quad\quad\quad\quad\quad\quad\quad\quad\quad~+ \left.\left\|\left(\text{Softmax}\left(\frac{\widehat{\mathbf{q}}_i^t [\widehat{X}^t W_K]^\top}{\sqrt{d}}\right) - \text{Softmax}\left(\frac{\widehat{\mathbf{q}}_i^t [\widehat{K}^t]^\top}{\sqrt{d}}\right)\right)\widehat{V}^t\right\|_2\right).
\end{align*}
Due to the property of softmax (each row sums to 1), we have
\begin{align*}
\left\|\text{Softmax}\left(\frac{\widehat{\mathbf{q}}_i^t [\widehat{X}^t W_K]^\top}{\sqrt{d}}\right)\left((\widehat{X}^t W_V) - \widehat{V}^t\right)\right\|_2 \leq \|\widehat{X}^t W_V - \widehat{V}^t\|_F.
\end{align*}
Besides, the Lipschitz continuity of softmax (Lemma~\ref{lem:softmax_lip}) and $\|\widehat{\mathbf{x}}_i^t\|_2 = 1$ lead to
\begin{align*}
& \left\|\left(\text{Softmax}\left(\frac{\widehat{\mathbf{q}}_i^t [\widehat{X}^t W_K]^\top}{\sqrt{d}}\right) - \text{Softmax}\left(\frac{\widehat{\mathbf{q}}_i^t [\widehat{K}^t]^\top}{\sqrt{d}}\right)\right)\widehat{V}^t\right\|_2 \\
& \leq \|\widehat{X}^t W_V\|_F \left\|\text{Softmax}\left(\frac{\widehat{\mathbf{q}}_i^t [\widehat{X}^t W_K]^\top}{\sqrt{d}}\right) - \text{Softmax}\left(\frac{\widehat{\mathbf{q}}_i^t [\widehat{K}^t]^\top}{\sqrt{d}}\right)\right\|_2\\
& \leq \frac{\|W_V\|_2 \|\widehat{X}^t\|_F}{\sqrt{d}}\|\widehat{\mathbf{q}}_i^t(\widehat{X}^t W_K - \widehat{K}^t)^\top\|_2 \leq \frac{\|W_V\|_2 \|W_Q\|_2}{\sqrt{d}} \|\widehat{X}^t W_K - \widehat{K}^t\|_F.
\end{align*}
Thus, we can derive that
\begin{align}\label{eq:dist_to_err}
& \|p_i^t(\cdot\mid \widehat{X}^t) -  \widehat{p}_i^t(\cdot\mid \widehat{X}^t) \|_1 \\\nonumber
& \leq \|E\|_{2\rightarrow\infty} \|W_D\|_2 G_\sigma \|W_U\|_2 \|W_O\|_2\left(\|\widehat{X}^t W_V - \widehat{V}^t\|_F + \frac{\|W_V\|_2 \|W_Q\|_2}{\sqrt{d}}  \|\widehat{X}^t W_K - \widehat{K}^t\|_F\right).
\end{align}
Let $\cos(\phi) = \frac{\widehat{\mathbf{x}}_i^t W_Q^\top W_Q [\widehat{\mathbf{x}}_i^{t-1}]^\top}{\|\widehat{\mathbf{x}}_i^t W_Q\|_2\|\widehat{\mathbf{x}}_i^{t-1} W_Q\|_2}$ and $\cos(\theta) = \frac{\langle \widehat{\mathbf{x}}_i^t, \widehat{\mathbf{x}}_i^{t-1}\rangle}{\|\widehat{\mathbf{x}}_i^t\|_2\|\widehat{\mathbf{x}}_i^{t-1} \|_2}$. According to Theorem 2.4 in \citet{lin2012generalized}, we have
\begin{align*}
\tan(\phi/2) \geq \frac{\tan(\theta/2)}{\kappa(W_Q)},
\end{align*}
where $\kappa(W_Q) :=\frac{\sigma_{\max}(W_Q)}{\sigma_{\min}(W_Q)}$. Since $\tan^2(\theta/2) = \frac{1-\cos(\theta)}{1+\cos(\theta)}$, we have
\begin{align*}
\frac{1-\cos(\phi)}{1+\cos(\phi)} \geq \frac{1}{\kappa(W_Q)^2}\left(\frac{1-\cos(\theta)}{1+\cos(\theta)}\right) \quad \Rightarrow \quad \cos(\theta) \geq \frac{\cos(\phi) - K}{1-K\cos(\phi)},
\end{align*}
where $K:= \frac{\kappa(W_Q)^2 - 1}{\kappa(W_Q)^2 + 1}$. If $i$ is in the reuse indices $\mathcal{R}^t$ at timestep $t$, we have $\cos(\phi)\geq 1-\tau$ such that 
\begin{align*}
\cos(\theta) \geq \frac{1- \tau - K}{ 1 - K(1-\tau)} = \frac{2-\tau\left(\kappa(W_Q)^2+1\right)}{2+\tau\left(\kappa(W_Q)^2-1\right)}.
\end{align*}
Since $\widehat{X}$ is layer normalized such that $\|\widehat{\mathbf{x}}_i^t\|_2 =  \|\widehat{\mathbf{x}}_i^{t-1}\|_2 = \sqrt{d}$ such that 
\begin{align}\label{eq:x_diff}
\|\widehat{\mathbf{x}}_i^t - \widehat{\mathbf{x}}_i^{t-1}\|_2^2 = 2d(1-\cos(\theta) )\leq \frac{4d\tau \kappa(W_Q)^2 }{2+\tau (\kappa(W_Q)^2 - 1)}. 
\end{align}

We define $\Tilde{\tau} = \frac{2\tau \cdot d \cdot \kappa(W_Q)^2 }{2+\tau (\kappa(W_Q)^2 - 1)}$. According to the definition of $\delta_{i,t}$, the error due to reusing the key for $i \in \mathcal{R}^t$ is
\begin{align*}
\|\widehat{\mathbf{x}}_i^t W_K - \widehat{\mathbf{k}}_i^t\|_2 & = \|\widehat{\mathbf{x}}_i^t W_K - \widehat{x}_i^{t-\delta_{i,t}} W_K\|_2\\
& \leq \|W_K\|_2 \left(\|\widehat{\mathbf{x}}_i^t - \widehat{\mathbf{x}}_i^{t-1}\|_2  + \dotsc +  \|\widehat{\mathbf{x}}_i^{t-\delta_{i,t}+1} - \widehat{\mathbf{x}}_i^{t-\delta_{i,t}}\|_2 \right) \leq \|W_K\|_2 \delta_{i,t} \sqrt{2\Tilde{\tau} }.
\end{align*}
Similarly, we have $\|\widehat{\mathbf{x}}_i^t W_V - \widehat{\mathbf{v}}_i^t\|_2 \leq \|W_V\|_2 \delta_{i,t}\sqrt{2\Tilde{\tau} }$. Plugging these into \eqref{eq:dist_to_err} results in
\begin{align*}
& \|p_i^t(\cdot\mid \widehat{X}^t) -  \widehat{p}_i^t(\cdot\mid \widehat{X}^t) \|_1 \\
& \leq \|E\|_{2\rightarrow\infty} \|W_D\|_2 G_\sigma \|W_U\|_2 \|W_O\|_2\left(\|W_V\|_2  + \frac{\|W_V\|_2 \|W_Q\|_2}{\sqrt{d}}\right)\sqrt{2\Tilde{\tau} \sum_{i=1}^B \delta_{i,t}^2}.
\end{align*}

\end{proof}

\subsubsection{Bounding the O Reuse Error in Each Step}\label{sec:bound_oreuse}

Due to the Lipschitz continuity of the softmax operator (Lemma~\ref{lem:softmax_lip}), for each $i\in[n]$ we have
\begin{align*}
& \|p_i^t(\cdot\mid \widehat{X}^t) -  \widehat{p}_i^t(\cdot\mid \widehat{X}^t) \|_1 = \|\text{Softmax}(h_i^t(\widehat{X}^t)E^\top) - \text{Softmax}(\widehat{h}_i^t(\widehat{X}^t)E^\top)\|_1\\
& \leq \|(h_i^t(\widehat{X}^t) - \widehat{h}_i^t(\widehat{X}^t))E^\top\|_\infty \leq \|E\|_{2\rightarrow\infty}\|h_i^t(\widehat{X}^t) - \widehat{h}_i^t(\widehat{X}^t)\|_2\\
& = \|E\|_{2\rightarrow\infty}\left\|\sigma\left(\left(\text{Softmax}\left(\frac{\widehat{\mathbf{q}}_i^t [\widehat{K}^t]^\top}{\sqrt{d}}\right)\widehat{V}^t W_O\right)W_U\right) W_D - \sigma(\widehat{\mathbf{o}}_i^tW_U) W_D\right\|_2.
\end{align*}
Due to the $G_\sigma$-Lipschitz continuity of the nonlinear activation $\sigma$, we further have
\begin{align*}
& \|p_i^t(\cdot\mid \widehat{X}^t) -  \widehat{p}_i^t(\cdot\mid \widehat{X}^t) \|_1 \\
& \leq \|E\|_{2\rightarrow\infty} \|W_D\|_2 G_\sigma \|W_U\|_2 \left\|\text{Softmax}\left(\frac{\widehat{\mathbf{q}}_i^t [\widehat{K}^t]^\top}{\sqrt{d}}\right)\widehat{V}^t W_O - \widehat{\mathbf{o}}_i^t\right\|_2.
\end{align*}
According to the definition of $\delta_{i,t}$, the error due to reusing the key for $i \in \mathcal{R}^t$ is
\begin{align*}
& \|\text{Softmax}\left(\frac{\widehat{\mathbf{q}}_i^t [\widehat{K}_i^t]^\top}{\sqrt{d}}\right)\widehat{V}_i^t W_O - \widehat{\mathbf{o}}_i^t\|_2  = \left\|\text{Softmax}\left(\frac{\widehat{\mathbf{q}}_i^t [\widehat{K}^t]^\top}{\sqrt{d}}\right)\widehat{V}^t W_O  - \text{Softmax}\left(\frac{\widehat{\mathbf{q}}_i^{t-\delta_{i,t}} [\widehat{K}^{t-\delta_{i,t}}]^\top}{\sqrt{d}}\right)\widehat{V}^{t-\delta_{i,t}} W_O \right\|_2\\
& \leq \underbrace{\left\|\left(\text{Softmax}\left(\frac{\widehat{\mathbf{q}}_i^t [\widehat{K}^t]^\top}{\sqrt{d}}\right) - \text{Softmax}\left(\frac{\widehat{\mathbf{q}}_i^{t-\delta_{i,t}} [\widehat{K}^t]^\top}{\sqrt{d}}\right)\right)\widehat{V}^t W_O \right\|_2}_{\text{I.}} \\
& \quad\quad + \underbrace{\left\|\left( \text{Softmax}\left(\frac{\widehat{\mathbf{q}}_i^{t-\delta_{i,t}} [\widehat{K}^t]^\top}{\sqrt{d}}\right) - \text{Softmax}\left(\frac{\widehat{\mathbf{q}}_i^{t-\delta_{i,t}} [\widehat{K}^{t-\delta_{i,t}}]^\top}{\sqrt{d}}\right)\right)\widehat{V}^t W_O \right\|_2}_{\text{II.}} \\
& \quad\quad +  \underbrace{\left\| \text{Softmax}\left(\frac{\widehat{\mathbf{q}}_i^{t-\delta_{i,t}} [\widehat{K}^{t-\delta_{i,t}}]^\top}{\sqrt{d}}\right)(\widehat{V}^t - \widehat{V}^{t-\delta_{i,t}}) W_O \right\|_2}_{\text{III.}}.
\end{align*}
Since $\|\widehat{X}^t\|_F = \sqrt{\sum_{i=1}^B \|\widehat{x}_i^t\|_2^2} = \sqrt{B d}$, the softmax operation is Lipschitz continuous, and the l2 norm of vector after softmax operation is upper bounded by 1, we can upper bound the three terms above as
\begin{align*}
& \text{I.} \leq B\sqrt{d}\|W_O\|_F\|W_V\|_F \|W_K\|_F \|\widehat{\mathbf{q}}_i^t - \widehat{\mathbf{q}}_i^{t-\delta_{i,t}} \|_2 \leq B\sqrt{d}\|W_O\|_F\|W_V\|_F \|W_K\|_F \|W_Q\|_F \|\widehat{\mathbf{x}}_i^t - \widehat{\mathbf{x}}_i^{t-\delta_{i,t}}\|_2,\\
& \text{II.} \leq \sqrt{Bd}\|W_O\|_F\|W_V\|_F\|W_Q\|_F\|\widehat{K}^t - \widehat{K}^{t-\delta_{i,t}}\|_F\leq \sqrt{Bd}\|W_O\|_F\|W_V\|_F \|W_K\|_F \|W_Q\|_F \sqrt{\sum_{i=1}^B\|\widehat{\mathbf{x}}_i^t - \widehat{\mathbf{x}}_i^{t-\delta_{i,t}}\|_2^2},\\
& \text{III.} \leq \|W_O\|_F \|\widehat{V}^t - \widehat{V}^{t-\delta_{i,t}}\|_F \leq \|W_O\|_F\|W_V\|_F\sqrt{\sum_{i=1}^B\|\widehat{\mathbf{x}}_i^t - \widehat{\mathbf{x}}_i^{t-\delta_{i,t}}\|_2^2}.
\end{align*}
Note that (\ref{eq:x_diff}) still holds such that we can handle $\|\widehat{\mathbf{x}}_i^t - \widehat{\mathbf{x}}_i^{t-\delta_{i,t}}\|_2$ as in the proof of Lemma~\ref{lem:reuse_err}.

\newpage

\section{Additional Generation Results}
We provide full \ours generation results for LLaDA (in \autoref{tab:instruct}), LLaDA1.5 (in \autoref{tab:instruct-15}), and Dream (in \autoref{tab:dream}).

\begin{table*}[t]
\footnotesize
\setlength{\tabcolsep}{4pt}
\centering
\caption{Performance of LLaDA-Instruct-8B on four benchmarks.}
\label{tab:instruct}
\begin{tabular}{@{}l l c c@{}}
\toprule
\multirow{1}{*}{Benchmark} & \multirow{1}{*}{Method} 
& Flexible (\%) & Strict (\%) \\
\midrule
\multirow{7}{*}{\begin{tabular}[c]{@{}l@{}}GSM8K \\ \textit{5-shot}\end{tabular}}
& Baseline                        & $77.63 \pm 1.15$ & $39.73 \pm 1.35$ \\
& Prefix Cache                    & $78.47 \pm 1.13$ & $37.38 \pm 1.33$ \\
& Parallel                        & $78.24 \pm 1.14$ & $40.49 \pm 1.35$ \\
& Prefix Cache + Parallel          & $77.18 \pm 1.16$ & $37.68 \pm 1.33$ \\
& \cellcolor{gray!15} Prefix Cache + Parallel + \kours   & \cellcolor{gray!15} $78.24 \pm 1.14$ & \cellcolor{gray!15} $36.62 \pm 1.33$ \\
& Fast-dLLM             & $78.17 \pm 1.14$ & $37.91 \pm 1.34$ \\
& \cellcolor{gray!15} Fast-dLLM + \kours   & \cellcolor{gray!15} $78.47 \pm 1.13$ & \cellcolor{gray!15} $40.71 \pm 1.35$ \\
\midrule
\multirow{7}{*}{\begin{tabular}[c]{@{}l@{}}MATH \\ \textit{4-shot}\end{tabular}}
& Baseline                        & $7.92 \pm 0.38$ & $32.78 \pm 0.62$ \\
& Prefix Cache                    & $7.00 \pm 0.36$ & $32.72 \pm 0.62$ \\
& Parallel                        & $6.98 \pm 0.36$ & $32.64 \pm 0.62$ \\
& Prefix Cache + Parallel          & $6.94 \pm 0.36$ & $32.76 \pm 0.62$ \\
& \cellcolor{gray!15} Prefix Cache + Parallel + \kours    & \cellcolor{gray!15} $7.24 \pm 0.36$ & \cellcolor{gray!15} $32.62 \pm 0.62$ \\
& Fast-dLLM        & $8.06 \pm 0.38$ & $32.44 \pm 0.62$ \\
&\cellcolor{gray!15} Fast-dLLM + \kours       &\cellcolor{gray!15} $8.60 \pm 0.39$ &\cellcolor{gray!15} $32.08 \pm 0.62$ \\
\midrule
\multirow{7}{*}{\begin{tabular}[c]{@{}l@{}}HumanEval \\ \textit{0-shot}\end{tabular}}
& Baseline                        & 42.07 & -- \\
& Prefix Cache                    & 42.68 & -- \\
& Parallel                        & 43.90 & -- \\
& Prefix Cache + Parallel          & 42.68 & -- \\
& \cellcolor{gray!15} Prefix Cache + Parallel + \kours    & \cellcolor{gray!15} 43.29 & -- \\
& Fast-dLLM  & 35.98 & -- \\
& \cellcolor{gray!15} Fast-dLLM + \kours     & \cellcolor{gray!15} 34.76 & -- \\
\midrule
\multirow{7}{*}{\begin{tabular}[c]{@{}l@{}}MBPP \\ \textit{3-shot}\end{tabular}}
& Baseline                        & $30.40 \pm 2.10$ & -- \\
& Prefix Cache                    & $27.00 \pm 2.00$ & -- \\
& Parallel                        & $30.60 \pm 2.10$ & -- \\
& Prefix Cache + Parallel          & $27.00 \pm 2.00$ & -- \\
& \cellcolor{gray!15} Prefix Cache + Parallel + \kours & \cellcolor{gray!15} $26.20 \pm 2.00$ & -- \\
& Fast-dLLM & $25.80 \pm 2.00$ & -- \\
& \cellcolor{gray!15} Fast-dLLM + \kours &\cellcolor{gray!15} $26.60 \pm 2.00$ & -- \\
\midrule
\end{tabular}
\end{table*}

\begin{table*}[t]
\footnotesize
\setlength{\tabcolsep}{4pt}
\centering
\caption{Performance of LLaDA-1.5 on four benchmarks.}
\label{tab:instruct-15}
\begin{tabular}{@{}l l c c@{}}
\toprule
\multirow{1}{*}{Benchmark} & \multirow{1}{*}{Method} 
& Flexible (\%) & Strict (\%) \\
\midrule
\multirow{7}{*}{\begin{tabular}[c]{@{}l@{}}GSM8K \\ \textit{5-shot}\end{tabular}}
& Baseline                        & $77.48 \pm 1.15$ & $40.49 \pm 1.35$ \\
& Prefix Cache                    & $80.97 \pm 1.08$ & $60.27 \pm 1.35$ \\
& Parallel                        & $81.35 \pm 1.07$ & $65.81 \pm 1.31$ \\
& Prefix Cache + Parallel          & $80.29 \pm 1.10$ & $59.21 \pm 1.35$ \\
& \cellcolor{gray!15}Prefix Cache + Parallel + \kours & \cellcolor{gray!15} $79.30 \pm 1.12$ & \cellcolor{gray!15} $62.55 \pm 1.33$ \\
& Fast-dLLM                        & $80.74 \pm 1.09$ & $59.06 \pm 1.35$ \\
&\cellcolor{gray!15} Fast-dLLM + \kours        & \cellcolor{gray!15}$80.44 \pm 1.09$ &\cellcolor{gray!15} $66.72 \pm 1.30$ \\
\midrule
\multirow{7}{*}{\begin{tabular}[c]{@{}l@{}}MATH \\ \textit{4-shot}\end{tabular}}
& Baseline                        & $33.34 \pm 0.63$ & $33.66 \pm 0.62$ \\
& Prefix Cache                    & $32.32 \pm 0.63$ & $33.12 \pm 0.62$ \\
& Parallel                        & $32.94 \pm 0.63$ & $33.58 \pm 0.62$ \\
& Prefix Cache + Parallel          & $32.00 \pm 0.63$ & $32.84 \pm 0.62$ \\
& \cellcolor{gray!15} Prefix Cache + Parallel + \kours &\cellcolor{gray!15} $32.54 \pm 0.63$ &\cellcolor{gray!15} $33.28 \pm 0.62$ \\
& Fast-dLLM                        & $32.00 \pm 0.63$ & $32.88 \pm 0.63$ \\
& \cellcolor{gray!15} Fast-dLLM + \kours         & \cellcolor{gray!15} $31.36 \pm 0.62$ &\cellcolor{gray!15} $32.50 \pm 0.62$ \\
\midrule
\multirow{7}{*}{\begin{tabular}[c]{@{}l@{}}HumanEval \\ \textit{0-shot}\end{tabular}}
& Baseline                        & 43.29 & -- \\
& Prefix Cache                    & 40.24 & -- \\
& Parallel                        & 43.29 & -- \\
& Prefix Cache + Parallel          & 39.02 & -- \\
& \cellcolor{gray!15} Prefix Cache + Parallel + \kours &\cellcolor{gray!15} 39.63 & -- \\
& Fast-dLLM                        & 35.98 & -- \\
& \cellcolor{gray!15} Fast-dLLM + \kours         & \cellcolor{gray!15} 35.98 & -- \\
\midrule
\multirow{7}{*}{\begin{tabular}[c]{@{}l@{}}MBPP \\ \textit{3-shot}\end{tabular}}
& Baseline                        & $39.00 \pm 2.18$ & -- \\
& Prefix Cache                    & $37.80 \pm 2.17$ & -- \\
& Parallel                        & $38.60 \pm 2.18$ & -- \\
& Prefix Cache + Parallel          & $38.00 \pm 2.17$ & -- \\
& \cellcolor{gray!15} Prefix Cache + Parallel + \kours & \cellcolor{gray!15} $38.80 \pm 2.18$ & -- \\
& Fast-dLLM                        & $35.80 \pm 2.15$ & -- \\
& \cellcolor{gray!15} Fast-dLLM + \kours        & \cellcolor{gray!15} $33.80 \pm 2.12$ & -- \\
\midrule
\end{tabular}
\end{table*}

\begin{table*}[t]
\footnotesize
\setlength{\tabcolsep}{4pt}
\centering
\caption{Performance of Dream-Base-7B on four benchmarks.}
\label{tab:dream}
\begin{tabular}{@{}l l c c@{}}
\toprule
\multirow{1}{*}{Benchmark} & \multirow{1}{*}{Method} 
& Flexible (\%) & Strict (\%) \\
\midrule
\multirow{7}{*}{\begin{tabular}[c]{@{}l@{}}GSM8K \\ \textit{5-shot}\end{tabular}}
& Baseline                        & $75.13 \pm 1.19$ & $74.75 \pm 1.20$ \\
& Prefix Cache                    & $75.82 \pm 1.18$ & $75.51 \pm 1.18$ \\
& Parallel                        & $73.16 \pm 1.22$ & $73.24 \pm 1.22$ \\
& Prefix Cache + Parallel          & $75.06 \pm 1.19$ & $74.98 \pm 1.19$ \\
& \cellcolor{gray!15} Prefix Cache + Parallel + \kours & \cellcolor{gray!15}$74.45 \pm 1.20$ & \cellcolor{gray!15} $74.45 \pm 1.20$ \\
& Fast-dLLM                        & $73.46 \pm 1.22$ & $73.39 \pm 1.22$ \\
& \cellcolor{gray!15} Fast-dLLM + \kours          & \cellcolor{gray!15} $72.93 \pm 1.22$ &\cellcolor{gray!15} $72.86 \pm 1.22$ \\
\midrule
\multirow{7}{*}{\begin{tabular}[c]{@{}l@{}}MATH \\ \textit{4-shot}\end{tabular}}
& Baseline                        & $38.36 \pm 0.64$ & $34.46 \pm 0.63$ \\
& Prefix Cache                    & $37.08 \pm 0.64$ & $33.70 \pm 0.63$ \\
& Parallel                        & $37.76 \pm 0.64$ & $34.12 \pm 0.63$ \\
& Prefix Cache + Parallel          & $37.70 \pm 0.64$ & $34.04 \pm 0.63$ \\
&\cellcolor{gray!15} Prefix Cache + Parallel + \kours &\cellcolor{gray!15} $37.20 \pm 0.64$ &\cellcolor{gray!15} $33.26 \pm 0.62$ \\
& Fast-dLLM                        & $37.14 \pm 0.64$ & $33.58 \pm 0.63$ \\
& \cellcolor{gray!15}Fast-dLLM + \kours          & \cellcolor{gray!15}$35.90 \pm 0.63$ &\cellcolor{gray!15} $32.38 \pm 0.62$ \\
\midrule
\multirow{7}{*}{\begin{tabular}[c]{@{}l@{}}HumanEval \\ \textit{0-shot}\end{tabular}}
& Baseline                        & 53.05 & -- \\
& Prefix Cache                    & 53.05 & -- \\
& Parallel                        & 51.22 & -- \\
& Prefix Cache + Parallel          & 54.88 & -- \\
& \cellcolor{gray!15} Prefix Cache + Parallel + \kours & \cellcolor{gray!15} 44.51 & -- \\
& Fast-dLLM                        & 54.27 & -- \\
& \cellcolor{gray!15} Fast-dLLM + \kours          & \cellcolor{gray!15} 48.78 & -- \\
\midrule
\multirow{7}{*}{\begin{tabular}[c]{@{}l@{}}MBPP \\ \textit{3-shot}\end{tabular}}
& Baseline                        & $55.60 \pm 2.22$ & -- \\
& Prefix Cache                    & $52.40 \pm 2.24$ & -- \\
& Parallel                        & $53.60 \pm 2.23$ & -- \\
& Prefix Cache + Parallel          & $54.60 \pm 2.23$ & -- \\
 & \cellcolor{gray!15} Prefix Cache + Parallel + \kours &\cellcolor{gray!15} $54.40 \pm 2.23$ & -- \\
& Fast-dLLM                        & $53.00 \pm 2.23$ & -- \\
& \cellcolor{gray!15} Fast-dLLM + \kours         & \cellcolor{gray!15} $54.20 \pm 2.23$ & -- \\
\midrule
\end{tabular}
\end{table*}

\end{document}